\documentclass[twoside]{article}

\usepackage[preprint]{aistats2027}
\usepackage[round]{natbib}

\usepackage{pifont}

\newcommand{\cmark}{\ding{51}} 
\newcommand{\xmark}{\ding{55}}

\usepackage{amsmath,amsthm,amssymb,mathtools,booktabs,microtype,enumitem}
\usepackage{algorithm,algpseudocode}
\usepackage[table]{xcolor}         
\usepackage{makecell}
\usepackage{dsfont}
\usepackage{enumitem}
\newtheorem{theorem}{Theorem}

\newtheorem{lemma}{Lemma}
\newtheorem{definition}{Definition}
\newtheorem{remark}{Remark}
\newtheorem{assump}{Assumption}

\newcommand{\ones}{\mathds{1}}

\newcommand{\cA}{\mathcal{A}}
\newcommand{\cB}{\mathcal{B}}

\newcommand{\cE}{\mathcal{E}}
\newcommand{\cF}{\mathcal{F}}

\newcommand{\cM}{\mathcal{M}}

\newcommand{\cO}{\mathcal{O}}

\newcommand{\hv}{\hat{v}^{k,h}}

\newcommand{\hV}{\tilde{V}^{k,h}}
\newcommand{\hg}{\hat{g}^{k,h}}
\newcommand{\hF}{\hat{F}^{k,h}}
\newcommand{\hX}{\hat{X}^{k,h}}

\newcommand{\hA}{\hat{A}^{k,h}}
\newcommand{\hB}{\hat{B}^{k,h}}

\newcommand{\Ql}{Q_\textnormal{lin}}

\newcommand{\inter}[1]{\ifblank{#1}{\inter}{\textnormal{int}(#1)}}

\newcommand{\tmix}{\tau_\textnormal{mix}}

\newcommand{\bE}{\mathbb{E}}

\newcommand{\bP}{\mathbb{P}}
\newcommand{\bR}{\mathbb{R}}

\newcommand{\Entk}{\mathcal{E}_{\textnormal{NTK}}}

\newcommand{\cT}{\mathcal{T}}
\newcommand{\MLMC}{\textnormal{MLMC}}

\newcommand{\KL}{\textnormal{KL}}

\newcommand{\SA}{|\cS| |\cA|}
\newcommand{\piS}{\pi^*}
\newcommand{\wS}{w^*}

\newcommand{\Tm}{T_\textnormal{max}}
\newcommand{\Hm}{H_\textnormal{max}}

\newcommand{\xS}{x^*}

\newcommand{\bH}{H_w}

\newcommand{\cL}{\mathcal{L}}
\newcommand{\cP}{\mathcal{P}}

\newcommand{\cR}{\mathcal{R}}
\newcommand{\cS}{\mathcal{S}}

\newcommand{\bbeta}{\bar{\beta}}

\newcommand{\epsb}{\varepsilon_\textnormal{bias}}
\newcommand{\epsap}{\varepsilon_\textnormal{app}}

\newcommand{\TV}{\textnormal{TV}}

\newcommand{\floor}[1]{\left\lfloor #1 \right\rfloor}

\newcommand{\Jp}{J^\pi}

\newcommand{\pth}{\pi_\theta}
\newcommand{\Jth}{J^\theta}

\newcommand{\nuth}{\nu^\theta}
\newcommand{\Ath}{A^\theta}
\newcommand{\Qth}{Q^\theta}

\newcommand{\pthk}{\pi_{\theta_k}}
\newcommand{\Jthk}{J^{\theta_k}}

\newcommand{\Qthk}{Q^{\theta_k}}

\newcommand{\Vc}{\textnormal{Vec}}

\begin{document}

%
\runningtitle{}

%
\runningauthor{}

\twocolumn[

\aistatstitle{Hierarchical Multilevel Monte Carlo for Order-Optimal Neural Actor-Critic in Average-Reward CMDPs}

\aistatsauthor{ Ankur Naskar \And Vaneet Aggarwal}

\aistatsaddress{ Indian Institute of Science, India \\ ankurnaskar@iisc.ac.in \And  Purdue University, USA \\ vaneet@purdue.edu} ]

\begin{abstract}
  Constrained Markov Decision Processes (CMDPs) provide a natural framework for reinforcement learning in safety-critical applications, where agents maximize long-term reward while satisfying long-term constraints. Although primal-dual actor-critic methods with linear critics are well understood, extending order-optimal convergence guarantees to neural critics in average-reward CMDPs has remained open. The main challenge is a bias--cost tradeoff: under Neural Tangent Kernel (NTK) analysis, reducing critic bias requires long optimization runs, making repeated critic estimation costly within primal-dual methods. We address this bottleneck with a hierarchical Multilevel Monte Carlo (MLMC) critic that randomizes both trajectory lengths and critic-optimization horizons. The estimator inherits the bias of a long critic run at polylogarithmic expected sample cost. Combined with primal-dual natural actor--critic updates, it yields $\widetilde{\mathcal O}(T^{-1/2})$ bounds on the averaged expected optimality gap and constraint violation, up to policy and critic approximation errors, using $\widetilde{\mathcal O}(T)$ expected transitions. These conditional expectation bounds hold with high probability over critic initialization, with logarithmic dependence on the inverse failure parameter. To our knowledge, these are the first order-optimal sample-complexity guarantees, modulo approximation errors, for average-reward CMDPs with general policy parameterization and neural critics, without requiring mixing-time knowledge. Our analysis also establishes the first such guarantees for unconstrained average-reward MDPs.
\end{abstract}

\section{Introduction}
\label{s: introduction}

Reinforcement learning (RL) has emerged as a powerful framework for sequential decision-making, with recent successes largely driven by deep neural networks. Many applications require agents not only to maximize long-term performance but also to satisfy resource or operational constraints, including transportation~\citep{al2019deeppool, haliem2021distributed}, communication networks~\citep{panju2021queueing}, robotics~\citep{chen2023option, gonzalez2023asap}, and healthcare~\citep{tamboli2024reinforced}. Such problems are naturally modeled as average-reward constrained Markov Decision Processes (CMDPs), where the objective is to maximize the long-run average reward while satisfying long-term safety or resource constraints. Compared with discounted formulations, the average-reward criterion directly captures steady-state performance and is therefore particularly suitable for continuing decision-making tasks.

Primal-dual policy-gradient and actor-critic methods are widely studied for CMDPs due to their scalability and compatibility with function approximation. The actor uses critic-estimated gradients to update a parameterized policy, while a dual variable enforces long-term constraints. With linear critics and known mixing time, recent analyses establish $\widetilde{\mathcal O}(T^{-1/2})$ convergence under general policy parameterization, up to approximation errors~\citep{xu2026global}. Related guarantees hold for discounted CMDPs~\citep{mondal2024sample}. The dependence on $T$ matches the minimax sample-complexity lower bound for average-reward reinforcement learning~\citep{jin2021towards}.

Despite these advances, extending order-optimal guarantees to neural critics in average-reward CMDPs remains open. Neural networks offer flexible function approximation and are widely used in reinforcement learning. Neural Tangent Kernel (NTK) analyses have enabled finite-sample guarantees for neural actor-critic methods by approximating neural critics as linear models near initialization. These techniques yield order-optimal convergence for unconstrained discounted MDPs~\citep{ganesh2025order}. However, for average-reward CMDPs, the best existing neural-critic rate is $\widetilde{\mathcal O}(T^{-1/4})$~\citep{Satheesh2026GlobalCO}, leaving a substantial gap from the $\widetilde{\mathcal O}(T^{-1/2})$ rate attainable with linear critics~\citep{xu2026global}. These rates hold up to approximation errors in the respective analyses.

The difficulty is not merely the use of neural networks, but a bias-cost tradeoff introduced by neural critic estimation. Under the NTK framework, the projection required to maintain the linear approximation couples the squared critic-bias bound to the mean-squared error (MSE) bound. Consequently, obtaining a smaller critic-bias bound requires running the critic optimization for substantially more iterations. In contrast, the primal-dual actor-critic algorithm operates within a nested-loop architecture, where each outer primal-dual update can only afford a limited amount of critic computation while maintaining the overall sample budget. At fixed confidence, the convergence analysis can tolerate a critic MSE of order $\widetilde{\cO}(1)$, whereas attaining the $\widetilde{\mathcal O}(T^{-1/2})$ rate requires the squared critic bias to be of order $\widetilde{\mathcal O}(T^{-1})$, apart from approximation errors. This creates a bias-cost bottleneck in existing analyses of primal-dual neural actor-critic methods.

In this work, we address this bottleneck by introducing a hierarchical Multilevel Monte Carlo (MLMC) neural critic. Existing trajectory-level MLMC techniques randomize trajectory lengths to reduce sampling bias under Markovian observations. Such estimators inherit the bias of a long trajectory while incurring only a logarithmic expected sampling cost. However, they do not address the optimization bias arising from finite neural critic training. Our key insight is to introduce a second MLMC layer that randomizes the optimization horizon itself, thereby reducing critic-optimization bias in addition to trajectory-sampling bias.

\begin{table*}[t]
\caption{
Comparison with related works under general policy parameterizations. The table indicates the type of critic employed (\xmark\ for linear vs \cmark\ for neural), the convergence and constraint-violation rates achieved (with ``N/A'' denoting unconstrained MDPs), and whether knowledge of the mixing time is required.
}
\label{tab:comparison}
\centering
\small
\begin{tabular}{lcccc}
\toprule
\textbf{Reference}
& \textbf{\makecell{Neural\\Critic}}
& \textbf{\makecell{Convergence\\Rate}}
& \textbf{\makecell{Constraint\\Violation}}
& \textbf{\makecell{Unknown\\Mixing Time}} \\
\midrule
\citet{pmlr-v267-ganesh25b}
& \xmark
& $\widetilde{\mathcal O}(T^{-1/2})$
& N/A
& \cmark \\
\citet{xu2026global}
& \xmark
& $\widetilde{\cO}(T^{-1/2+\varepsilon})$
& $\widetilde{\cO}(T^{-1/2+\varepsilon})$
& \cmark \\
\citet{xu2026global}
& \xmark
& $\widetilde{\cO}(T^{-1/2})$
& $\widetilde{\cO}(T^{-1/2})$
& \xmark \\
\citet{ganesh2025order}
& \cmark
& $\widetilde{\cO}(T^{-1/2})$
& N/A
& \xmark \\
\citet{Satheesh2026GlobalCO}
& \cmark
& $\widetilde{\cO}(T^{-1/4})$
& $\widetilde{\cO}(T^{-1/4})$
& \cmark \\
\rowcolor{gray!20}
\textbf{Our Work}
& \cmark
& $\widetilde{\cO}(T^{-1/2})$
& $\widetilde{\cO}(T^{-1/2})$
& \cmark \\
\bottomrule
\end{tabular}
\end{table*}

Our main contributions are summarized as follows.
\begin{itemize}[topsep=0pt, partopsep=0pt, itemsep=4pt, parsep=0pt]
    \item \textbf{Resolving the Neural Critic Bias-Cost Tradeoff:} We identify a bias-cost bottleneck in existing analyses of primal-dual actor-critic methods with neural critics. To address it, we introduce a \emph{hierarchical} MLMC neural critic that randomizes both trajectory lengths and critic-optimization horizons. Our estimator inherits the bias of a vanilla critic run for $\Theta(T)$ iterations while requiring only polylogarithmic expected sample cost.
    \item \textbf{Order-Optimal Sample Complexity with Neural Critics:} Building on our hierarchical MLMC critic, we establish $\widetilde{\mathcal O}(T^{-1/2})$ bounds on the averaged expected optimality gap and constraint violation, up to policy and critic approximation errors, using $\widetilde{\mathcal O}(T)$ expected transitions. These conditional expectation bounds hold with high probability over critic initialization, with logarithmic dependence on the inverse failure parameter. To our knowledge, these are the first order-optimal guarantees, modulo approximation errors, for primal-dual Natural Actor-Critic with neural critics and general policy parameterization in infinite-horizon average-reward CMDPs. Our analysis also establishes the first such guarantees for unconstrained average-reward MDPs.
    \item \textbf{Mixing-Time-Free Algorithm:}
    Our algorithm requires no knowledge of the mixing time of the underlying CMDP, although the theoretical bounds depend on it. This removes the need for a mixing-time oracle.
\end{itemize}

\subsection{Related Works}

\paragraph{Average-Reward Constrained MDPs:}
CMDPs, introduced by~\citet{altman2021constrained}, underpin safe reinforcement learning. Primal-dual methods support function approximation~\citep{ding2020natural, paternain2019constrained}, with finite-sample guarantees in tabular and function-approximation settings~\citep{xu2021crpo, chen2022learning, bai2023achieving}. \citet{agarwal2022concave} and \citet{ghosh2023achieving} achieved zero constraint violation in model-based and linear model-free settings, respectively. Under general policy parameterization, \citet{bai2024learning} obtained a suboptimal rate, while \citet{xu2026global} achieved $\widetilde{\cO}(T^{-1/2})$ with known mixing time, up to approximation errors.

\paragraph{Actor-Critic Methods with Neural Critics:}
Early finite-sample analyses used linear critics~\citep{khodadadian2021finite, cayci2021linear}, while NTK theory~\citep{jacot2018neural} enabled neural-critic guarantees for discounted MDPs~\citep{fu2020single, tian2023convergence, cayci2024finite, gaur2024closing, ganesh2025order}. Up to approximation errors, \citet{ganesh2025order} achieved $\widetilde{\cO}(T^{-1/2})$ for unconstrained discounted MDPs, whereas \citet{Satheesh2026GlobalCO} obtained $\widetilde{\cO}(T^{-1/4})$, the best existing neural-critic rate for average-reward CMDPs.

\paragraph{MLMC-Based Estimation:} Existing analyses under Markovian sampling often rely on sample thinning~\citep{gaur2024closing, ganesh2025orderoptimal}, which requires knowledge of the mixing time. Recent works have instead leveraged Multilevel Monte Carlo (MLMC) methods~\citep{blanchet2015unbiased, beznosikov2023first,pmlr-v235-patel24b, pmlr-v267-ganesh25b} to construct bias-controlled estimators via randomized trajectory lengths, thereby eliminating the need for data dropping and mixing-time estimates. In the CMDP setting,~\citet{xu2026global} used MLMC to mitigate the bias-cost tradeoff in primal-dual natural actor-critic algorithms with linear critics.

\section{Problem Setup}
\label{s: problem.setup}

We consider an infinite-horizon average-reward constrained Markov Decision Process (CMDP) $\cM:= (\cS,\cA,\cP,r,c,\rho),$ where $\cS$ and $\cA$ are finite state and action spaces, $\cP:\cS\times\cA\to\Delta(\cS)$ is the transition kernel,\footnote{$\Delta(U)$ denotes the probability simplex over a finite set $U.$} $r:\cS\times\cA\to[0,1]$ and $c:\cS\times\cA\to[-1,1]$ are reward and cost functions, respectively, and $\rho\in\Delta(\cS)$ is the initial state distribution.
A stationary policy $\pi:\cS\to\Delta(\cA)$ generates a trajectory $(s_t,a_t)_{t=0}^\infty,$ with $s_0\sim\rho,$ $a_t\sim\pi(\cdot|s_t),$ and $s_{t+1}\sim\cP(\cdot|s_t,a_t).$

We make the following assumption on $\cM.$
\begin{assump}[\textbf{Ergodicity}]\label{a: ergodicity}
    For every stationary policy $\pi,$ the Markov chain $(s_t)$ is irreducible and aperiodic.
\end{assump}
Under Assumption~\ref{a: ergodicity}, each policy $\pi$ induces a unique stationary distribution $d^\pi,$ independent of $\rho,$ satisfying $d^\pi(s)>0$ for all $s\in\cS$ and $(d^\pi)^\top\cP^\pi=(d^\pi)^\top,$ where $\cP^\pi(s,s'):=\sum_a\pi(a|s)\cP(s'|s,a).$

\begin{definition}[\textbf{Mixing Time}]\label{def: mixing.time}
    The mixing time of $\cM$ under policy $\pi$ is
    \[
        \tmix^\pi:= \min\left\{t\geq 1:
        \max_s\left\|(\cP^\pi)^t(s,\cdot)-d^\pi\right\|_\TV
        \leq 1/4\right\},
    \]
    where $t$ is an integer and $\|\cdot\|_\TV$ is the total variation norm. The uniform mixing time is $\tmix:=\sup_\pi\tmix^\pi.$
\end{definition}
The mixing time quantifies convergence to stationarity. Assumption~\ref{a: ergodicity} implies $\tmix<\infty$~\citep{Satheesh2026GlobalCO}.

For a policy $\pi,$ define its average reward $\Jp_r$ and average constraint cost $\Jp_c$ as
\begin{equation}\label{e: def.avg.utility}
    \Jp_u := \lim_{T\to\infty}\frac{1}{T}\bE\left[\sum_{t=0}^{T-1} u(s_t,a_t) \middle| s_0 \sim \rho \right],
\end{equation}
for $u\in\{r,c\},$ where the expectation is over trajectories generated by $\pi.$ Under Assumption~\ref{a: ergodicity}, $\Jp_u=\bE_{(s,a)\sim\nu^\pi}[u(s,a)],$ where $\nu^\pi(s,a):=d^\pi(s)\pi(a|s)$ is the stationary occupancy measure. 

For large state spaces, a common approach is to parameterize the policy as $\pi_\theta,$ where $\theta\in\Theta\subset\bR^d$ and $d\ll\SA.$ We write $\Jth_u=J^{\pi_\theta}_u$ and $\nuth=\nu^{\pi_\theta},$ and let $\bE_\theta[\cdot]$ denote expectation over $(s,a,s',a')$ with $(s,a)\sim\nuth,$ $s'\sim\cP(\cdot|s,a),$ and $a'\sim\pth(\cdot|s').$ Given these quantities, our objective is to solve
\begin{align}\label{e: goal.param}
    \max_{\theta\in\Theta} \ \Jth_r
    \quad \textnormal{s.t. } \ \Jth_c\geq 0.
\end{align}
We impose the following strict feasibility assumption~\citep{wei2022provably, bai2024learning}.
\begin{assump}[\textbf{Slater Condition}]\label{a: slater}
    There exists $\delta\in (0,1)$ and $\bar{\theta}\in \Theta$ such that $J_c^{\bar{\theta}}\geq \delta.$
\end{assump}

\subsection{Primal-Dual Natural Policy Gradient}
\label{s: PD.NPG}

The primal-dual approach reformulates the constrained optimization~\eqref{e: goal.param} as the saddle-point problem
\begin{equation}\label{e: lagrange.function}
    \max_{\theta \in \Theta}\min_{\lambda\geq 0} \cL(\theta, \lambda) := \Jth_r + \lambda \Jth_c,
\end{equation}
where $\cL$ is called the Lagrangian, and $\lambda$ is referred to as the dual variable. A common gradient-based scheme for~\eqref{e: lagrange.function} is Natural Policy Gradient (NPG), defined by the updates
\begin{equation}\label{e: lagrange.natural.PG}
    \begin{aligned}
        &\theta_{k+1} = \theta_k + \alpha F(\theta_k)^\dagger\nabla_\theta \cL(\theta_k, \lambda_k), 
        \\
        & \quad \lambda_{k+1} = \Pi_{[0,2/\delta]}\left(\lambda_k - \beta \Jthk_c \right).
    \end{aligned}
\end{equation}
Here, $F(\theta)$ is the Fisher information matrix defined as 
\begin{equation}\label{e: Fischer.matrix}
    F(\theta) := \bE_\theta\left[ \nabla_\theta \ln \pi_\theta(a|s)\left(\nabla_\theta \ln \pi_\theta(a|s)\right)^\top \right],
\end{equation}
where $\dagger$ denotes the Moore-Penrose pseudoinverse, $\alpha$ and $\beta$ are the primal and dual stepsizes, respectively, $\delta$ is the constant from Assumption~\ref{a: slater}, and  $\Pi_A$ denotes projection onto a set $A\subseteq\bR.$

The NPG vector $\wS_k:= F(\theta_k)^\dagger\nabla_\theta\cL(\theta_k,\lambda_k)$ is obtained by minimizing the quadratic objective
\begin{multline}\label{e: compatibility.func}
    f(\theta_k,\lambda_k, w) 
    := \frac{1}{2} w^\top F(\theta_k)w
    - w^\top \nabla_\theta \cL(\theta_k,\lambda_k).
\end{multline}
This is achieved by applying gradient descent with 
\begin{equation}\label{e: NPG.gradient}
    \nabla_w f(\theta_k, \lambda_k, w) 
    = F(\theta_k)w - \left( \nabla_\theta \Jthk_r + \lambda_k\nabla_\theta \Jthk_c\right),
\end{equation}
which in turn requires computing $\nabla_\theta \Jth_u, u\in \{r,c\}.$ By the policy gradient theorem \citep{sutton1999policy}, for each $u\in \{r,c\},$ we have 
\begin{equation}\label{e: policy.gradient}
    \nabla_\theta \Jth_u = \bE_\theta\left[ \Ath_u(s,a)\ \nabla_\theta \ln \pth(a|s) \right].
\end{equation}
Here, $\Ath_u\in \bR^{\SA}$ denotes the advantage function defined as
\begin{equation}\label{e: advantage}
    \Ath_u(s,a) := \Qth_u(s,a) - \bE_{a'\sim\pth(\cdot|s)}\left[\Qth_u(s,a')\right],
\end{equation}
and $\Qth_u\in \bR^{\SA}$ is a solution to the Bellman equation
\begin{equation}\label{e: bellman}
    Q(s,a) = u(s,a) - \Jth_u 
    + \bE_\theta\left[Q(s',a')\right].
\end{equation}
Such a function is unique (up to additive constants) and is called the state-action value function for the policy $\pth.$

Consequently, implementing~\eqref{e: lagrange.natural.PG} hinges on estimating $(\Jthk_u,\Qthk_u)$ for $u\in\{r,c\}$. We now describe the critic estimator used for this purpose.

\subsection{Neural Critic Estimation}
\label{s: NN.critic}

Given a policy $\pth,$ $\Jth_u$ is a scalar expectation, whereas estimating $\Qth_u$ involves function estimation over state-action pairs. We employ a neural critic for its expressive function approximation. Let $\phi : \cS\times\cA\to\bR^{d}$ be a feature map such that $\|\phi(s,a)\|\leq 1,$ for every state-action pair $(s,a).$ Given matrices $W_u^{(1)}\in \bR^{m\times d}$ and $\{W_u^{(\ell)}\}_{\ell=2}^L\subset \bR^{m\times m},$ we let the vector

    $\qquad \zeta_u := \left[\Vc^\top\big(W_u^{(1)}\big), \ldots, \Vc^\top\big(W_u^{(L)}\big)\right]^\top$

serve as our critic parameter.\footnote{$\Vc(\cdot)$ denotes vectorization by stacking columns.} These matrices consist of trainable weights in an $L$-layer feedforward neural network defined as
\begin{equation}
    x^{(\ell)}_u(s,a) := \frac{1}{\sqrt{m}}\ \sigma\left( W_u^{(\ell)} x^{(\ell-1)}_u(s,a) \right), \ \ell \in [L],
\end{equation}
where $x^{(0)}_u(s,a) := \phi(s,a)$ and $\sigma$ is an element-wise activation function. For each $u\in\{r,c\},$ the output of this neural network lets us approximate $\Qth_u(s,a)$ as
\begin{equation}
    Q(s,a,\zeta_u) := \frac{1}{\sqrt{m}}\ b_u^\top x_u^{(L)}(s,a),
\end{equation}
where $b_u\in\bR^m$ is fixed during training. Following~\citet{ganesh2025order, Satheesh2026GlobalCO}, draw
$[W_u^{(\ell)}(0)]_{ij}\sim\mathcal N(0,1)$ and
$[b_u]_i\sim\mathrm{Unif}\{-1,+1\}$ mutually independently and independently of sampling randomness.

Throughout the paper, we make the following assumption.
\begin{assump}[\textbf{Smooth Activation}]\label{a: sooth.activation}
    The activation function $\sigma$ is $L_1$-Lipschitz and $L_2$-smooth.
    %
    %
\end{assump}
\begin{remark}
    Assumption~\ref{a: sooth.activation} is satisfied by commonly used smooth approximations of ReLU such as Sigmoid, ELU, and GeLU ~\citep{Hendrycks2016GaussianEL, Clevert2015FastAA}. 
\end{remark}
For subsequent analysis, we rely on the Neural Tangent Kernel (NTK) regime, which requires restricting the critic parameter $\zeta_u$ to a ball $\cB$ of radius $R$ around the initialization $\zeta_0.$ This lets us approximate the neural network output  $Q(\cdot,\cdot,\zeta)$ by the linear function class
\begin{equation}\label{e: linearized.Q.function}
    \cF_{\cB,m} 
    := \Big\{ \Ql(\cdot,\cdot,\zeta) : \zeta \in \cB \Big\},
\end{equation}
where $\Ql(\cdot, \cdot, \zeta_0)$ is defined by
\begin{equation}\label{e: linear.approx}
    \Ql(\cdot, \cdot, \zeta):=Q(\cdot, \cdot, \zeta_0) + (\zeta-\zeta_0)^\top\nabla_\zeta Q(\cdot,\cdot,\zeta_0).
\end{equation}

The critic initializations and output weights are sampled once and retained across all outer iterations. All vanilla-critic runs for signal $u$ start from the same $\zeta_{u,0}$.

%

\section{Algorithm}
\label{s: algorithm}

We present our hierarchical MLMC-based primal-dual Natural Actor-Critic algorithm (HiMLMC-PD-NAC) for solving~\eqref{e: goal.param}, with pseudocode in Algorithm~\ref{alg:HiMLMC}. Since the transition kernel and occupancy measures are unknown, we use sample-based estimates.

The algorithm updates the primal-dual pair $(\theta_k,\lambda_k)$ as
\begin{equation}\label{e: actor.update}
    \begin{aligned}
        & \qquad \theta_{k+1} = \theta_k + \alpha w_k
        \\
        & \lambda_{k+1} = \Pi_{[0,2/\delta]}\left( \lambda_k - \beta \eta_{c,k} \right),
    \end{aligned}
\end{equation}
where $\alpha,\beta,$ and $\delta$ are as in~\eqref{e: lagrange.natural.PG}, and $w_k$ and $\eta_{c,k}$ approximate the NPG vector $\wS_k$ and average constraint cost $\Jthk_c$ under $\pthk,$ respectively. We run~\eqref{e: actor.update} for $K$ iterations, yielding $\{(\theta_k,\lambda_k)\}_{k<K}.$

\subsection{Critic Estimator}
\label{s: critic.subroutine}

At each outer-loop iteration $0\leq k<K,$ we estimate $(\Jthk_u,\Qthk_u)$ by $(\eta_{u,k,\MLMC},Q_{u,k,\MLMC})$ using a hierarchical two-layer MLMC scheme with outer- and inner-layer truncation parameters $\Hm$ and $\Tm,$ respectively. The outer layer controls critic-optimization bias by combining outputs $\xi^k_{u,H}$ of a vanilla-critic subroutine run for different numbers $H$ of iterations. The inner layer controls trajectory-sampling bias in gradient estimation within this subroutine.

\paragraph{Outer MLMC layer:} Sample $H_k\sim\textnormal{Geom}(1/2)$ and set
\[
    h_k:=H_0\big[1+\ones_{\{H_02^{H_k}\leq\Hm\}}(2^{H_k}-1)\big],
\]
where $H_0\leq\Hm$ is a predefined base horizon. Define the critic estimator function
\begin{align}\label{e: function .critic.MLMC}
    & Q_{u,k,\MLMC}(\cdot,\cdot)
    := Q\big(\cdot,\cdot,\zeta^k_{u,H_0}\big)
    + \ones_{\{H_02^{H_k}\leq\Hm\}}
    \nonumber\\
    & \quad\times 2^{H_k}\Big[
    Q\big(\cdot,\cdot,\zeta^k_{u,h_k}\big)
    - Q\big(\cdot,\cdot,\zeta^k_{u,h_k/2}\big)\Big]
\end{align}
and the MLMC critic parameter
\begin{align}\label{e: output.critic.MLMC}
    & \xi_{u,k,\MLMC}
    = \left[\eta_{u,k,\MLMC},(\zeta_{u,k,\MLMC})^\top\right]^\top
    \nonumber\\
    & := \xi^k_{u,H_0}
    + \ones_{\{H_02^{H_k}\leq\Hm\}}2^{H_k}
    \Big[\xi^k_{u,h_k}-\xi^k_{u,h_k/2}\Big],
\end{align}
where $\xi^k_{u,H}$ is the vanilla-critic output after $H$ iterations. The correction terms are evaluated only when $H_02^{H_k}\leq\Hm;$ otherwise, only the base run of $H_0$ iterations is performed.

\paragraph{Inner MLMC layer:}
The vanilla-critic subroutine estimates the average reward/cost and value function by optimizing the  objectives
\begin{align*}
    \cR_u(\theta_k,\eta)
    &:= \frac{1}{2}\bE_{\theta_k}\left[(\eta-u(s,a))^2\right],\\
    \cE_u(\theta_k,\zeta)
    &:= \frac{1}{2}\bE_{\theta_k}\left[
    \big(Q(s,a,\zeta)-\Qthk_u(s,a)\big)^2\right].
\end{align*}
We use a stochastic approximation update combining mean estimation and temporal-difference (TD) learning. For a run of length $H,$ let
\[
    \xi^k_{u,h}
    = \big[\eta^k_{u,h},(\zeta^k_{u,h})^\top\big]^\top,
    \qquad 0\leq h<H.
\]
To control Markovian sampling bias, independently sample
$U_{k,h}\sim\textnormal{Geom}(1/2)$ on $\{1,2,\ldots\}$ and set
\begin{equation}\label{e: traj.length}
    \ell_{k,h}:=
    1+\ones_{\{2^{U_{k,h}}\leq\Tm\}}(2^{U_{k,h}}-1).
\end{equation}
Given a trajectory
$\cT^{k,h}:=(s^{k,h}_t,a^{k,h}_t)_{t=0}^{\ell_{k,h}}$
generated under $\pthk,$ define the sample-average direction
$\hv_{u,T}:=\frac{1}{T}\sum_{t=0}^{T-1}
\hv_u(s^{k,h}_t,a^{k,h}_t,s^{k,h}_{t+1},a^{k,h}_{t+1}),$
where
\begin{equation}\label{e: critic.gradient}
    \hv_u(s,a,s',a')
    := \begin{bmatrix}
        c_\gamma\hat{\nabla}_{\eta}\cR^{k,h}_u(s,a)\\
        \hat{\nabla}_{\zeta}\cE^{k,h}_u(s,a,s',a')
    \end{bmatrix},
\end{equation}
$c_\gamma>0$ is a scaling parameter, and
\begin{equation*}
    \begin{aligned}
        &\hat{\nabla}_{\eta}\cR^{k,h}_u(s,a)
        :=\eta^k_{u,h}-u(s,a),\\
        &\hat{\nabla}_{\zeta}\cE^{k,h}_u(s,a,s',a')
        :=\nabla_\zeta Q(s,a,\zeta_0)\\
        &\quad\times\big(
        \eta^k_{u,h}-u(s,a)+Q(s,a,\zeta^k_{u,h})
        -Q(s',a',\zeta^k_{u,h})\big).
    \end{aligned}
\end{equation*}
Here, $\hat{\nabla}_{\zeta}\cE^{k,h}_u$ denotes the TD direction using the gradient at initialization. The inner MLMC estimator is
\begin{multline}\label{e: ctitic.inner.MLMC}
    \hv_{u,\MLMC}:=\hv_{u,1}
    +\ones_{\{2^{U_{k,h}}\leq\Tm\}}\\
    \times 2^{U_{k,h}}
    \left(\hv_{u,2^{U_{k,h}}}
    -\hv_{u,2^{U_{k,h}-1}}\right).
\end{multline}
We perform $H$ updates
\begin{equation}\label{e: critic.update}
    \xi^k_{u,h+1}
    =\Pi_\cB\left(
    \xi^k_{u,h}-\gamma_{\xi,H}\hv_{u,\MLMC}\right),
\end{equation}
where $\gamma_{\xi,H}$ depends on $H,$ and $\Pi_\cB$ projects only the $\zeta$ component onto the NTK ball $\cB.$ After completing $H$ iterations for each required horizon $H\in\{H_0,h_k/2,h_k\},$ we have the quantities required in~\eqref{e: output.critic.MLMC}. 

\begin{remark}\label{rem: shared.critic.runs}
    The required vanilla-critic outputs can be computed in a single loop of length $h_k,$ sharing sampled trajectories across separate critic copies, each updated with its own stepsize until its prescribed horizon (see Algorithm~\ref{alg:critic}).
\end{remark}

\subsection{NPG Estimator}
\label{e: NPG.estimator}

At each outer iteration $0\leq k<K,$ we estimate $\wS_k$ by approximately minimizing $f(\theta_k,\lambda_k,w)$ in~\eqref{e: compatibility.func}. At each inner iteration, independently sample $U_{k,h}\sim\textnormal{Geom}(1/2)$ and generate a trajectory $\cT^{k,h}:=(s^{k,h}_t,a^{k,h}_t)_{t=0}^{\ell_{k,h}}$ under $\pthk,$ with $\ell_{k,h}$ defined in~\eqref{e: traj.length}. The $T$-sample estimators of $F(\theta_k)$ and $\nabla_\theta\Jthk_u$ are
\begin{align*}
    \hF_T &:= \frac{1}{T}\sum_{t=0}^{T-1}
    \hF(s^{k,h}_t,a^{k,h}_t),
    \\
    \hg_{u,T} &:= \frac{1}{T}\sum_{t=0}^{T-1}
    \hg_u(s^{k,h}_t,a^{k,h}_t,\bar a^{k,h}_t,
    s^{k,h}_{t+1},a^{k,h}_{t+1}),
\end{align*}
where $\bar a^{k,h}_t\sim\pthk(\cdot|s^{k,h}_t)$ is drawn independently of the trajectory conditional on $s^{k,h}_t,$ and
\begin{equation}\label{e: naive.NPG.estimator}
    \begin{aligned}
        &\hF(s,a):=\nabla_\theta\ln\pthk(a|s)
        \left(\nabla_\theta\ln\pthk(a|s)\right)^\top,
        \\
        &\hg_u(s,a,\bar a,s',a')
        :=\nabla_\theta\ln\pthk(a|s)\Big(
        u(s,a)
        \\
        & -\eta_{u,k,\MLMC} +Q_{u,k,\MLMC}(s',a')
        -Q_{u,k,\MLMC}(s,\bar a)\Big).
    \end{aligned}
\end{equation}
For $\hX\in\{\hF,\hg_r,\hg_c\},$ define the corresponding MLMC estimators
\begin{multline}\label{e: NPG.MLMC.estimator}
    \hX_\MLMC:=\hX_1
    +\ones_{\{2^{U_{k,h}}\leq\Tm\}}
    \\
    \times 2^{U_{k,h}}
    \Big(\hX_{2^{U_{k,h}}}-\hX_{2^{U_{k,h}-1}}\Big).
\end{multline}
Using these estimates, we perform the SGD update
\begin{multline}\label{e: NPG.update}
    w^k_{h+1}=w^k_h
    \\
    -\gamma_w\left[ \hF_\MLMC w^k_h -\big(\hg_{r,\MLMC}+\lambda_k\hg_{c,\MLMC}\big)\right].
\end{multline}
After $\bH$ iterations, we set $w_k:=w^k_{\bH}.$

\begin{remark}[\textbf{Expected Sample Cost}]
Since $H_k$ is independent of the inner-level draws $U_{k,h},$ the expected critic-estimation cost per outer iteration is
$\cO\!\left(H_0[1+\ln(\Hm/H_0)][1+\ln\Tm]\right)$
for $\Hm\geq H_0$ and $\Tm\geq1.$
The corresponding NPG-estimation cost is
$\cO\!\left(\bH[1+\ln\Tm]\right).$
These costs count sampled transitions, excluding neural-network computation and memory.
\end{remark}
\section{Assumptions and Main Results}
\label{s: algorithm}

Before we present our finite-time convergence bounds in Theorem~\ref{thm: main result}, we state the following assumptions.

\begin{assump}[\textbf{Score Function}]\label{a: score.function}
    There exist $G_1,$ $G_2>0,$ such that for all state-action pairs $(s,a),$ the score function $\theta\mapsto \nabla_\theta \ln\pth(a|s)$ is uniformly bounded by $G_1$~and $G_2$-Lipschitz.
\end{assump}

\begin{assump}[\textbf{Fisher Non-Degeneracy}]\label{a: Fisher.degenrate}
    There exists $\mu_F>0,$ such that for all $\theta\in \Theta,$ $F(\theta) \succeq \mu_F I.$ 
\end{assump}

\begin{assump}[\textbf{Policy Parametrization Error}]\label{a: FA.error.PG}
    There exists $\epsb>0$ such that for all $\theta\in \Theta$ and $\lambda\in [0,2/\delta],$ we have 
    \[
        L_{\nu^{\piS}}(\theta,\lambda, \wS_{\lambda,\theta})\leq \epsb,
    \]
    where $\piS$ and $\nu^{\piS}$  denote an optimal stationary policy\footnote{$\piS$ is a solution to the constrained optimization: $\max_{\pi} \Jp_r$ subject to $\Jp_c\geq 0.$} and its occupancy measure, respectively, $\wS_{\lambda,\theta}:=F(\theta)^\dagger\nabla_\theta\cL(\lambda,\theta),$ while the function $L_\nu(\theta, \lambda, w)$ is called the transferred compatibility function and is defined as
    \[
         \bE_{\nu}\left[ \left(w^\top \nabla_\theta \ln \pth(a|s) - \Ath_r(s,a) - \lambda \Ath_c(s,a)\right)^2\right].
    \]
\end{assump}

\begin{remark}
    These assumptions are standard in PG analyses~\citep{liu2020improved, papini2018stochastic, Xu2020Sample, fatkhullin2023stochastic, ding2025convergence}. Softmax with bounded features satisfies Assumption~\ref{a: score.function}. Tabular softmax with one fixed logit per state satisfies Assumption~\ref{a: Fisher.degenrate} when action and stationary state probabilities are uniformly bounded away from zero, ensuring strong convexity of $f(\theta,\lambda,\cdot)$. It also has zero compatibility error in Assumption~\ref{a: FA.error.PG}.
\end{remark}

\begin{assump}[\textbf{Critic Approximation Error}]\label{a: neural.critic.error}
    There exists a finite $\epsap>0,$ such that for every  $u\in \{r,c\},$ and $\theta\in \Theta,$  we have
    \[
         \bE_{\theta}\left[ \left(Q^{\theta}_u(s,a) - \left(\Pi_{\cF_{\cB,m}}Q^{\theta}_u\right)(s,a) \right)^2\right] \leq \epsap,
    \]
    where $\Pi_{\cF_{\cB,m}}$ denotes projection onto the linear function class defined in \eqref{e: linearized.Q.function}.
\end{assump}

\begin{remark}
    Assumption~\ref{a: neural.critic.error} bounds the linearized critic's approximation error, which vanishes under realizability. Related assumptions appear in~\citet[Assumption~2.5]{gaur2024closing} and~\citet[Assumption~2]{cayci2024finite}.
\end{remark}

Let $\mathcal I_0$ denote the initial weights and output signs
of both critics. Under Assumptions~\ref{a: ergodicity}--\ref{a: neural.critic.error},
we have the following result.

\begin{theorem}\label{thm: main result}
Choose $\gamma_{\xi,H}=\frac{2\ln T}{\mu_\phi H}$
with suitable $\mu_\phi>0,$ and
$\gamma_w=\frac{2\ln T}{\mu_F\bH}.$
Set $m=\Hm,$ $R=\varTheta(\ln T),$
$\alpha=\beta=\frac{1}{\sqrt K},$
$H_0=1+\floor{4\ln T},$
$\Tm=\Hm=K=\varTheta(T),$ and
$\bH=\varTheta(\ln^2 T).$
Then, with probability at least
$(1-2\varepsilon-4Le^{-Cm})$ over $\mathcal I_0,$
\begin{align*}
    J^{\piS}_r-\frac{1}{K}\sum_{k=0}^{K-1}
    \bE[\Jthk_r]
    = \widetilde{\cO}\left(
    \frac{\ln(1/\varepsilon)}{\sqrt T}+\sqrt{\epsb}+\sqrt{\epsap}\right)
    \\
    -\frac{1}{K}\sum_{k=0}^{K-1}
    \bE[\Jthk_c]
    = \widetilde{\cO}\left(
    \frac{\ln(1/\varepsilon)}{\sqrt T}+\sqrt{\epsb}+\sqrt{\epsap}\right),
\end{align*}
where the expectation is conditioned on $\mathcal{I}_0,$ $\piS$ is defined in Assumption~\ref{a: FA.error.PG},
$\varepsilon\in(0,1/2)$ is a confidence parameter,
and $C>0$ is defined in Lemma~\ref{app.lem: NTK.bounds} in Appendix~\ref{appendix: technical.lemmas}, and $L$ is the network depth.
The $\widetilde{\cO}$ notation suppresses logarithmic
dependence on $T.$
\end{theorem}

\begin{remark}[\textbf{Optimal Global Convergence}]
    Theorem~\ref{thm: main result} gives a
    $\widetilde{\cO}(T^{-1/2})$ bound on the averaged expected optimality gap, up to $\sqrt{\epsb}+\sqrt{\epsap}$, with high probability over initialization. The total expected sample cost is $\cO\!\left(K(\bH+H_0[1+\ln(\Hm/H_0)])[1+\ln\Tm]\right) =\widetilde{\cO}(T),$ yielding order-optimal sample dependence up to approximation errors.
\end{remark}

\begin{remark}[\textbf{Constraint Violation Rate}]
    Under the same conditioning, the averaged expected constraint violation admits a $\widetilde{\cO}(T^{-1/2})$ bound up to $\sqrt{\epsb}+\sqrt{\epsap}$, matching the rate established with linear critics and known mixing time~\citep{xu2026global}.
\end{remark}

\begin{remark}[\textbf{Mixing-Time Knowledge}]
    The parameter choices in Theorem~\ref{thm: main result} require no knowledge of the mixing time of the underlying CMDP.
\end{remark}


\section{Proof Sketch}
\label{s : prrof.outlne}

We now outline the proof of Theorem~\ref{thm: main result}. Our proof consists of the following key steps.
\begin{enumerate}[nosep]
    \item Use the dual update in~\eqref{e: actor.update} to express the \textbf{optimality gap and constraint violation} as a combination of the Lagrangian error and critic bounds.
    \item Bound the \textbf{Lagrangian error} using the primal update in \eqref{e: actor.update} in terms of the NPG bias and MSE. 
    \item Exploit the linear structure of the NPG update~\eqref{e: NPG.update} together with MLMC bounds for the Fisher matrix and policy-gradient estimates to control the \textbf{NPG bias and MSE} in terms of the critic parameter.
    \item Bound the critic bias and MSE: 
    \begin{enumerate}[nosep]
        \item Bound the \textbf{vanilla critic}: decompose the vanilla critic update~\eqref{e: critic.update} into a perturbed linear recursion, then control the linear components using the inner MLMC layer~\eqref{e: ctitic.inner.MLMC}, and the nonlinear perturbation
        via the NTK-regime.
        \item Bound the \textbf{MLMC critic}: combine the vanilla-critic bounds with the outer MLMC layer~\eqref{e: output.critic.MLMC}.
    \end{enumerate}
\end{enumerate}
Throughout subsequent analysis, we suppress conditioning on $\mathcal I_0$ and write $\bE_k[\cdot]:=\bE[\cdot\mid\cF_k],$ where $\cF_k$ is generated by the initializations and all randomness preceding outer iteration $k$. Let $\Entk$ denote the initialization event on which NTK bounds from Lemma~\ref{app.lem: NTK.bounds} (see Appendix~\ref{appendix: technical.lemmas}) hold uniformly over both the reward- and cost-critics' NTK balls.

\paragraph{Step 1. Optimality Gap and Constraint Violation:} We follow the proof strategy of~\citet{pmlr-v267-ganesh25b},
adapted to primal-dual methods with linear critics
by~\citet{xu2026global}. Define
$\cL(\piS,\lambda):=J^{\piS}_r+\lambda J^{\piS}_c,$
where $\piS$ is an optimal feasible stationary policy.
The dual update~\eqref{e: actor.update} reduces both performance bounds to the averaged Lagrangian gap and critic errors, as follows.
\begin{lemma}[\textbf{Optimality Gap}]\label{lem: optimality.gap.bound}
    Under Assumption~\ref{a: slater}, we have
        \begin{align*}
            &\frac{1}{K}\sum_{k=0}^{K-1}
            \bE[J^{\piS}_r-\Jthk_r]
            \leq \frac{1}{K}\sum_{k=0}^{K-1}
            \bE[\cL(\piS,\lambda_k)-\cL(\theta_k,\lambda_k)]
            \\
            &\quad+\cO\Bigg(\beta+\frac{1}{K}\sum_{k=0}^{K-1}\Big[\beta\bE\|\xi_{c,k,\MLMC}-\xi^*_{c,k}\|^2
            \\
            & \qquad\qquad\qquad  \qquad\qquad+\bE\|\bE_k[\xi_{c,k,\MLMC}]-\xi^*_{c,k}\| \Big]\Bigg).
        \end{align*}
\end{lemma}
\begin{lemma}[\textbf{Constraint Violation Rate}]
\label{lem: constraint.violation.rate}
    Under Assumption~\ref{a: slater}, we have
    \begin{align*}
        &-\frac{1}{K}\sum_{k=0}^{K-1}\bE[\Jthk_c]
        \\
        & =\cO\Bigg(\beta+\frac{1}{\beta K}+\frac{1}{K}\sum_{k=0}^{K-1}\Big[\beta\bE\|\xi_{c,k,\MLMC}-\xi^*_{c,k}\|^2
        \\
        & \qquad \qquad +\bE\|\bE_k[\xi_{c,k,\MLMC}]-\xi^*_{c,k}\|\Big]
        \\
        &\qquad\qquad \qquad+\frac{1}{K}\sum_{k=0}^{K-1}
        \bE[\cL(\piS,\lambda_k)-\cL(\theta_k,\lambda_k)]\Bigg).
    \end{align*}
\end{lemma}
Note that MSEs are scaled by $\beta$, whereas the biases are not. Detailed proofs appear in Appendices~\ref{appendix: opt.gap}
and~\ref{appendix: constraint.violation}, respectively.

\paragraph{Step 2. Lagrangian Error:}
We next bound the Lagrangian error in terms of the NPG bias and MSE.
\begin{lemma}\label{lem: Lagrange.error}
    Under Assumptions~\ref{a: slater} and~\ref{a: score.function}--\ref{a: FA.error.PG}, the updates~\eqref{e: actor.update} satisfy
    \begin{align*}
        &\frac{1}{K}\sum_{k=0}^{K-1}\bE\left[\cL(\piS,\lambda_k)-\cL(\theta_k,\lambda_k)\right] 
        \\
        & =\cO\Bigg(\sqrt{\epsb}+\alpha + \frac{1}{\alpha K} 
        \\
        & \qquad + \frac{1}{K}\sum_{k=0}^{K-1}
        \left[\bE\|\bE_k[w_k]-\wS_k\|
        +\alpha\bE\|w_k-\wS_k\|^2\right]\Bigg).
    \end{align*}
\end{lemma}
The proof appears in Appendix~\ref{appendix: lagrange}.

\paragraph{Step 3. NPG Estimator Analysis:}
The NPG update~\eqref{e: NPG.update} is a stochastic linear recursion with target drift
$(\nabla_\theta\Jthk_r+\lambda_k\nabla_\theta\Jthk_c)-F(\theta_k)w.$ Its bias and MSE are controlled by those of the Fisher and PG estimates. For the Fisher estimate, trajectory-level MLMC controls sampling bias,
yielding squared bias $\cO(\tmix/\Tm)$ and MSE logarithmic in $\Tm.$ The PG estimate additionally incurs neural-critic error, which we control on the NTK event using the linearization $\Ql(\cdot,\cdot,\zeta)\in\cF_{\cB,m}$ defined in~\eqref{e: linear.approx}. 

Let $\xi^*_{u,k}:=[\Jthk_u,(\zeta^*_{u,k})^\top]^\top,$
where $\zeta^*_{u,k}$ is a minimum-norm solution of following projected Bellman equation:
\begin{multline}\label{e: proj.bellman}
    \bE_{\theta_k}\Big[\nabla_\zeta Q(s,a,\zeta_0)
    \big(\Jthk_u-u(s,a)
    \\
    +\Ql(s,a,\zeta)-\Ql(s',a',\zeta)\big)\Big]=0.
\end{multline}
\begin{theorem}[\textbf{NPG Bounds}]\label{thm: NPG.bounds}
    Let Assumptions~\ref{a: ergodicity}--\ref{a: neural.critic.error} hold and
    choose $\gamma_w=\frac{2\ln T}{\mu_F\bH}.$ Then, on the NTK event $\Entk,$ we have
    \begin{align*}
        & \|\bE_k[w_k]-\wS_k\|^2
        =\widetilde{\cO}\Bigg(
        \frac{\tmix}{\Tm}+\frac{\ln(\Hm/\varepsilon)}{m}+\epsap 
        \\
        & +\delta_Q^2 + \ln\left(\Hm/\varepsilon\right)\sum_{u\in\{r,c\}}
        \|\bE_k[\xi_{u,k,\MLMC}]-\xi^*_{u,k}\|^2\Bigg),
        \\
        & \bE_k\|w_k-\wS_k\|^2 =\widetilde{\cO}\Bigg(\tmix+\frac{\ln(\Hm/\varepsilon)}{m} + \epsap
        \\
        &+\sigma_Q^2 + \ln\left(\Hm/\varepsilon\right)\sum_{u\in\{r,c\}}
        \bE_k\|\xi_{u,k,\MLMC}-\xi^*_{u,k}\|^2\Bigg),
    \end{align*}
    where, 
    \begin{align*}
        \delta_Q^2
        &:=\sum_u\big\|\bE_k[
        Q_{u,k,\MLMC}-\Ql(\cdot,\cdot,\zeta_{u,k,\MLMC})]\big\|^2,
        \\
        \sigma_Q^2
        &:=\sum_u\bE_k\big\|
        Q_{u,k,\MLMC}-\Ql(\cdot,\cdot,\zeta_{u,k,\MLMC})\big\|^2.
    \end{align*}
\end{theorem}

Inside the NTK ball, the neural critic $Q(\cdot,\cdot,\zeta)$ decomposes into the linear critic $\Ql(\cdot,\cdot,\zeta)$ and a nonlinear remainder. This nonlinear remainder yields the $\widetilde{\cO}(\ln(\Hm/\varepsilon)/m)$ term in Theorem~\ref{thm: NPG.bounds}. Whereas, $\epsap$ reflects the approximation capacity of the linear critic. 

Consequently, the analysis is reduced to controlling the critic bias and MSE. Deferring the proof of Theorem~\ref{thm: NPG.bounds} to Appendix~\ref{appendix: NPG.analysis}, we now establish the critic bounds in the following step.

\paragraph{Step 4. Critic Estimator Analysis:} Recall the hierarchical MLMC scheme introduced in Section~\ref{s: critic.subroutine}. Since the outer layer combines outputs of the vanilla-critic subroutine corresponding to different iteration budgets, as in~\eqref{e: output.critic.MLMC}, we begin by analyzing the vanilla-critic $\xi^k_{u,H}$ generated by~\eqref{e: critic.update} with an arbitrary iteration budget $H.$ 

The NTK projection $\Pi_{\cB}$ introduces an additional error in the bias analysis of  $\xi^k_{u,H}.$ This projection error is controlled by the probability with which iterates escape the NTK ball, which, in turn, is bounded in terms of the vanilla-critic MSE ~\cite[see][Appendix A.2]{ganesh2025order}. To bound the vanilla-critic MSE, we decompose~\eqref{e: critic.update} as a perturbed linear recursion.
Under the NTK regime, the linear critic $\Ql(\cdot,\cdot,\zeta)$
approximates the neural critic up to an error $\widetilde{\cO}(m^{-1/2})$ (see Lemma 6). By leveraging this linear approximation, we can write~\eqref{e: critic.update} as:
\begin{multline*}
    \xi^k_{u,h+1} 
    = \Pi_\cB\Big( \xi^k_{u,h} 
    + \gamma_{\xi,H}\big[\hB_{u,\MLMC} - \hA_{\MLMC}\ \xi^k_{u,h}\big]
    \\+ \gamma_{\xi,H}\hV_{u,\MLMC} \Big),
\end{multline*}
where the linear component $(\hB_{u,\MLMC}-\hA_{\MLMC}\ \xi)$ is obtained by replacing $Q(\cdot,\cdot,\zeta)$ with $\Ql(\cdot,\cdot,\zeta)$ in~\eqref{e: critic.gradient} and using the inner-layer MLMC construction defined in~\eqref{e: ctitic.inner.MLMC}. Collecting the remainder yields the nonlinear component $\hV_{u,\MLMC} $. Under the NTK regime, we have $\|\hV_{u,\MLMC}\|=\widetilde{\cO}(m^{-1/2}),$
while standard MLMC bounds apply on the linear component. This results in the following bound for the vanilla critic (detailed proof in Appendix~\ref{appendix: critic.analysis}).
\begin{lemma}[\textbf{Vanilla-Critic Bounds}]\label{lem: vanilla.critic.bound}
    Let $\xi^k_{u,H}$ be generated by~\eqref{e: critic.update} after $H$ iterations and let  Assumptions~\ref{a: ergodicity}--\ref{a: neural.critic.error} hold. Choose $\gamma_{\xi,H}=\frac{2\ln T}{\mu_\phi H},$ for sufficiently small $\mu_\phi,$ and let $c_\gamma$ be sufficiently large. Then, for all $u\in\{r,c\},$ we have on the NTK event $\Entk,$
    \begin{multline*}
        \bE_k\left\|\xi^k_{u,H} - \xi^*_{u,k} \right\|^2 
        \\
        = \widetilde{\cO}\left(\frac{1}{T^2} +\frac{\tmix}{\Tm} + \frac{\tmix}{H}+ \frac{\tmix\ln(\Hm/\varepsilon)}{m}\right).
    \end{multline*}
\end{lemma}

\begin{remark}[\textbf{Bias-Cost Tradeoff}]\label{rem: bias.cost.tradeoff}
    Owing to the NTK projection $\Pi_{\cB}$, our squared critic-bias bound inherits the MSE's $\widetilde{\cO}(H^{-1})$ dependence, at fixed confidence and apart from sampling and approximation terms. In contrast, a linear critic yields squared bias $\widetilde{\cO}(\Tm^{-1})$~\citep[Lemma~4.8]{xu2026global}. Our primal-dual scheme requires many outer iterations to control constraint violations, making long critic runs at every iteration costly. Directly using $\xi^k_{u,H}$ thus creates a bias-cost bottleneck.
\end{remark}

We overcome the above bias-cost tradeoff by exploiting the outer MLMC layer defined by~\eqref{e: function .critic.MLMC} and \eqref{e: output.critic.MLMC}. Recall that the iteration budget for the vanilla subroutine is set to $h_k := H_0\big(1 + \ones_{\{H_02^{H_k}\leq \Hm\}}(2^{H_k}-1)\big)$ with  $H_k\sim\textnormal{Geom}(1/2).$ A telescoping-sum argument then shows that
\begin{multline*}
    \bE\big[\xi_{u,k,\MLMC}\big]
    = \bE\left[\xi^k_{u,H_0 2^{\floor{\log_2(\Hm/H_0)}}}\right] \quad \text{and}
    \\
    \bE[h_k]= H_0\left(
    \floor{\log_2(\Hm/H_0)}
    +2^{-\floor{\log_2(\Hm/H_0)}}
    \right).
\end{multline*}
Therefore, while the MLMC estimator inherits the bias of a vanilla critic updated for $\Hm$ iterations, the expected number of vanilla iterations is only $\cO(H_0\ln\Hm).$ Using a similar argument for the MSE, we prove the following bounds in Appendix~\ref{appendix: critic.analysis}.
\begin{theorem}[\textbf{MLMC-Critic Bounds}]\label{thm: critic.bounds.MLMC}
    Consider the setting of Lemma~\ref{lem: vanilla.critic.bound}. Then, the critic function $Q_{u,k,\MLMC}$ satisfies $\delta^2_Q = \widetilde{\cO}\left(\ln(\Hm/\varepsilon)/m\right)$ and $\sigma^2_Q = \widetilde{\cO}\left(\Hm\ln(\Hm/\varepsilon)/H_0m\right).$
    %
    %
    Moreover, the critic parameter $\xi_{u,k,\MLMC}$ satisfies:
    \begin{align*}
        &\bE_k\!\left\|\xi_{u,k,\MLMC}-\xi^*_{u,k}\right\|^2
        \\
        & \qquad\qquad = \widetilde{\cO}\!\left(
        \frac{\tmix}{H_0}+\frac{\Hm}{H_0 m}+ \frac{\ln(\Hm/\varepsilon)}{m}\right),
        \\
        & \left\|\bE_k\xi_{u,k,\MLMC} - \xi^*_{u,k}\right\|^2
        \\
        & \qquad\qquad = \widetilde{\cO}\!\left( \frac{\tmix}{\Tm} + \frac{\tmix}{\Hm} + \frac{\tmix\ln(\Hm/\varepsilon)}{m}\right).
    \end{align*}
\end{theorem}

Setting $H_0=1+\floor{4\ln T},$ $m=\Hm$ results in a critic MSE of order $\widetilde{\cO}(\ln(1/\varepsilon))$ and a critic bias of order $\widetilde{\cO}\big(\Tm^{-1/2} + \ln(1/\varepsilon)\Hm^{-1/2} \big).$

Finally, in Appendix~\ref{appendix: main.result.proof}, we combine Lemmas~\ref{lem: optimality.gap.bound}--\ref{lem: Lagrange.error} with Theorems~\ref{thm: NPG.bounds} and~\ref{thm: critic.bounds.MLMC}, and apply the prescribed parameter choices to complete the proof of Theorem~\ref{thm: main result}. \hfill$\square$

\section{Conclusion}
\label{s: conclusions}

Our hierarchical MLMC critic addresses the bias-cost bottleneck by randomizing trajectory and optimization horizons. Combined with primal-dual Natural Actor-Critic, it yields $\widetilde{\cO}(T^{-1/2})$ bounds on averaged expected optimality gap and constraint violation, up to approximation errors, with
high probability over initialization, using $\widetilde{\cO}(T)$ expected transitions without a mixing-time oracle.






\bibliographystyle{apalike} 
\bibliography{references}

@article{ke2024improved,
  title={An improved finite-time analysis of temporal difference learning with deep neural networks},
  author={Ke, Zhifa and Wen, Zaiwen and Zhang, Junyu},
  journal={arXiv preprint arXiv:2405.04017},
  year={2024}
}

@InProceedings{pmlr-v267-ganesh25b,
  title = 	 {A Sharper Global Convergence Analysis for Average Reward Reinforcement Learning via an Actor-Critic Approach},
  author =       {Ganesh, Swetha and Mondal, Washim Uddin and Aggarwal, Vaneet},
  booktitle = 	 {Proceedings of the 42nd International Conference on Machine Learning},
  pages = 	 {18206--18227},
  year = 	 {2025}
 }

@inproceedings{Satheesh2026GlobalCO,
  author    = {Anirudh Satheesh and
               Pankaj Kumar Barman and
               Washim Uddin Mondal and
               Vaneet Aggarwal},
  title     = {Global Convergence of Average Reward Constrained MDPs with Neural Critic and General Policy Parameterization},
  booktitle = {Proceedings of the 42nd Conference on Uncertainty in Artificial Intelligence (UAI)},
  year      = {2026}
}

@article{cayci2021linear,
  title={Linear convergence of entropy-regularized natural policy gradient with linear function approximation},
  author={Cayci, Semih and He, Niao},
  journal={arXiv preprint arXiv:2106.04096},
  year={2021}
}

@inproceedings{fatkhullin2023stochastic,
  title={Stochastic policy gradient methods: Improved sample complexity for fisher-non-degenerate policies},
  author={Fatkhullin, Ilyas and Barakat, Anas and Kireeva, Anastasia and He, Niao},
  booktitle={International Conference on Machine Learning},
  pages={9827--9869},
  year={2023},

}

@article{
agarwal2022concave,
title={Concave Utility Reinforcement Learning with Zero-Constraint Violations},
author={Mridul Agarwal and Qinbo Bai and Vaneet Aggarwal},
journal={Transactions on Machine Learning Research},
issn={2835-8856},
year={2022},
url={https://openreview.net/forum?id=WXVkgkPXRk},
note={}
}

@inproceedings{
xu2026global,
title={Global Convergence for Average Reward Constrained {MDP}s with Primal-Dual Actor Critic Algorithm},
author={Yang Xu and Swetha Ganesh and Washim Uddin Mondal and Qinbo Bai and Vaneet Aggarwal},
booktitle={The Thirty-ninth Annual Conference on Neural Information Processing Systems},
year={2025},
url={https://openreview.net/forum?id=9I1XjEEtsh}
}

@inproceedings{bai2024regret,
  title={Regret analysis of policy gradient algorithm for infinite horizon average reward markov decision processes},
  author={Bai, Qinbo and Mondal, Washim Uddin and Aggarwal, Vaneet},
  booktitle={Proceedings of the AAAI Conference on Artificial Intelligence},
  volume={38},
  number={10},
  pages={10980--10988},
  year={2024}
}

@inproceedings{
ganesh2025orderoptimal,
title={Order-Optimal Regret with Novel Policy Gradient Approaches in Infinite Horizon Average Reward {MDP}s},
author={Swetha Ganesh and Washim Uddin Mondal and Vaneet Aggarwal},
booktitle={The 28th International Conference on Artificial Intelligence and Statistics},
year={2025},
url={https://openreview.net/forum?id=ZJwMfQ6W9P}
}

@article{beznosikov2023first,
  title={First order methods with markovian noise: from acceleration to variational inequalities},
  author={Beznosikov, Aleksandr and Samsonov, Sergey and Sheshukova, Marina and Gasnikov, Alexander and Naumov, Alexey and Moulines, Eric},
  journal={Advances in Neural Information Processing Systems},
  volume={36},
  pages={44820--44835},
  year={2023}
}

@InProceedings{pmlr-v235-patel24b,
  title = 	 {Towards Global Optimality for Practical Average Reward Reinforcement Learning without Mixing Time Oracles},
  author =       {Patel, Bhrij and Suttle, Wesley A and Koppel, Alec and Aggarwal, Vaneet and Sadler, Brian M. and Manocha, Dinesh and Bedi, Amrit},
  booktitle = 	 {Proceedings of the 41st International Conference on Machine Learning},
  pages = 	 {39889--39907},
  year = 	 {2024}
 }

@article{haliem2021distributed,
  title={A distributed model-free ride-sharing approach for joint matching, pricing, and dispatching using deep reinforcement learning},
  author={Haliem, Marina and Mani, Ganapathy and Aggarwal, Vaneet and Bhargava, Bharat},
  journal={IEEE Transactions on Intelligent Transportation Systems},
  volume={22},
  number={12},
  pages={7931--7942},
  year={2021},
  publisher={IEEE}
}

@inproceedings{chen2023option,
  title={Option-Aware Adversarial Inverse Reinforcement Learning for Robotic Control},
  author={Chen, Jiayu and Lan, Tian and Aggarwal, Vaneet},
  booktitle={Proceedings IEEE International Conference on Robotics and Automation},
  year={2023}
}

@article{panju2021queueing,
  title={Queueing theoretic models for uncoded and coded multicast wireless networks with caches},
  author={Panju, Mahadesh and Raghu, Ramkumar and Sharma, Vinod and Aggarwal, Vaneet and Ramachandran, Rajesh},
  journal={IEEE Transactions on Wireless Communications},
  volume={21},
  number={2},
  pages={1257--1271},
  year={2021},
  publisher={IEEE}
}

@article{tamboli2024reinforced,
  title={Reinforced sequential decision-making for sepsis treatment: The posnegdm framework with mortality classifier and transformer},
  author={Tamboli, Dipesh and Chen, Jiayu and Jotheeswaran, Kiran Pranesh and Yu, Denny and Aggarwal, Vaneet},
  journal={IEEE Journal of Biomedical and Health Informatics},
  volume={28},
  number={5},
  pages={3114--3122},
  year={2024},
  publisher={IEEE}
}

@article{al2019deeppool,
  title={Deeppool: Distributed model-free algorithm for ride-sharing using deep reinforcement learning},
  author={Al-Abbasi, Abubakr O and Ghosh, Arnob and Aggarwal, Vaneet},
  journal={IEEE Transactions on Intelligent Transportation Systems},
  volume={20},
  number={12},
  pages={4714--4727},
  year={2019},
  publisher={IEEE}
}

@article{gonzalez2023asap,
  title={Asap: A semi-autonomous precise system for telesurgery during communication delays},
  author={Gonzalez, Glebys and Balakuntala, Mythra and Agarwal, Mridul and Low, Tomas and Knoth, Bruce and Kirkpatrick, Andrew W and McKee, Jessica and Hager, Gregory and Aggarwal, Vaneet and Xue, Yexiang and others},
  journal={IEEE Transactions on Medical Robotics and Bionics},
  volume={5},
  number={1},
  pages={66--78},
  year={2023},
  publisher={IEEE}
}

@article{mondal2024sample,
  title={Sample-efficient constrained reinforcement learning with general parameterization},
  author={Mondal, Washim U and Aggarwal, Vaneet},
  journal={Advances in Neural Information Processing Systems},
  volume={37},
  pages={68380--68405},
  year={2024}
}

@article{ding2025convergence,
  title={Convergence and sample complexity of natural policy gradient primal-dual methods for constrained mdps},
  author={Ding, Dongsheng and Zhang, Kaiqing and Duan, Jiali and Basar, Tamer and Jovanovic, Mihailo R},
  journal={Journal of Machine Learning Research},
  volume={26},
  number={256},
  pages={1--76},
  year={2025}
}

@article{bai2024learning,
  title={Learning general parameterized policies for infinite horizon average reward constrained MDPs via primal-dual policy gradient algorithm},
  author={Bai, Qinbo and Mondal, Washim U and Aggarwal, Vaneet},
  journal={Advances in Neural Information Processing Systems},
  volume={37},
  pages={108566--108599},
  year={2024}
}

@article{liu2020improved,
  title={An improved analysis of (variance-reduced) policy gradient and natural policy gradient methods},
  author={Liu, Yanli and Zhang, Kaiqing and Basar, Tamer and Yin, Wotao},
  journal={Advances in Neural Information Processing Systems},
  volume={33},
  pages={7624--7636},
  year={2020}
}

@inproceedings{papini2018stochastic,
  title={Stochastic variance-reduced policy gradient},
  author={Papini, Matteo and Binaghi, Damiano and Canonaco, Giuseppe and Pirotta, Matteo and Restelli, Marcello},
  booktitle={International conference on machine learning},
  pages={4026--4035},
  year={2018},

}

@inproceedings{
Xu2020Sample,
title={Sample Efficient Policy Gradient Methods with Recursive Variance Reduction},
author={Pan Xu and Felicia Gao and Quanquan Gu},
booktitle={International Conference on Learning Representations},
year={2020},
url={https://openreview.net/forum?id=HJlxIJBFDr}
}

@inproceedings{gaur2024closing,
  title={Closing the Gap: Achieving Global Convergence (Last Iterate) of Actor-Critic under Markovian Sampling with Neural Network Parametrization},
  author={Gaur, Mudit and Bedi, Amrit and Wang, Di and Aggarwal, Vaneet},
  booktitle={Forty-first International Conference on Machine Learning},
  year={2024}
}

@inproceedings{wei2022provably,
  title={A provably-efficient model-free algorithm for infinite-horizon average-reward constrained markov decision processes},
  author={Wei, Honghao and Liu, Xin and Ying, Lei},
  booktitle={Proceedings of the AAAI Conference on Artificial Intelligence},
  volume={36},
  number={4},
  pages={3868--3876},
  year={2022}
}

@article{Clevert2015FastAA,
  title={Fast and accurate deep network learning by exponential linear units (elus)},
  author={Clevert, Djork-Arn{\'e} and Unterthiner, Thomas and Hochreiter, Sepp},
  journal={arXiv preprint arXiv:1511.07289},
  volume={4},
  number={5},
  pages={11},
  year={2015}
}

@article{Hendrycks2016GaussianEL,
  title={Gaussian error linear units (gelus)},
  author={Hendrycks, Dan and Gimpel, Kevin},
  journal={arXiv preprint arXiv:1606.08415},
  year={2016}
}

@article{sutton1999policy,
  title={Policy gradient methods for reinforcement learning with function approximation},
  author={Sutton, Richard S and McAllester, David and Singh, Satinder and Mansour, Yishay},
  journal={Advances in neural information processing systems},
  volume={12},
  year={1999}
}

@inproceedings{ganesh2025order,
  title={Order-optimal global convergence for actor-critic with general policy and neural critic parametrization},
  author={Ganesh, Swetha and Chen, Jiayu and Mondal, Washim Uddin and Aggarwal, Vaneet},
  booktitle={The 41st Conference on Uncertainty in Artificial Intelligence},
  year={2025}
}

@book{altman2021constrained,
  title={Constrained Markov decision processes},
  author={Altman, Eitan},
  year={2021},
  publisher={Routledge}
}

@article{paternain2019constrained,
  title={Constrained reinforcement learning has zero duality gap},
  author={Paternain, Santiago and Chamon, Luiz and Calvo-Fullana, Miguel and Ribeiro, Alejandro},
  journal={Advances in Neural Information Processing Systems},
  volume={32},
  year={2019}
}

@article{ding2020natural,
  title={Natural policy gradient primal-dual method for constrained markov decision processes},
  author={Ding, Dongsheng and Zhang, Kaiqing and Basar, Tamer and Jovanovic, Mihailo},
  journal={Advances in Neural Information Processing Systems},
  volume={33},
  pages={8378--8390},
  year={2020}
}

@inproceedings{xu2021crpo,
  title={Crpo: A new approach for safe reinforcement learning with convergence guarantee},
  author={Xu, Tengyu and Liang, Yingbin and Lan, Guanghui},
  booktitle={International Conference on Machine Learning},
  pages={11480--11491},
  year={2021},

}

@inproceedings{chen2022learning,
  title={Learning infinite-horizon average-reward Markov decision process with constraints},
  author={Chen, Liyu and Jain, Rahul and Luo, Haipeng},
  booktitle={International Conference on Machine Learning},
  pages={3246--3270},
  year={2022},

}

@inproceedings{khodadadian2021finite,
  title={Finite-sample analysis of off-policy natural actor-critic algorithm},
  author={Khodadadian, Sajad and Chen, Zaiwei and Maguluri, Siva Theja},
  booktitle={International Conference on Machine Learning},
  pages={5420--5431},
  year={2021},

}

@article{cayci2024finite,
  title={Finite-Time Analysis of Entropy-Regularized Neural Natural Actor-Critic Algorithm},
  author={Cayci, Semih and He, Niao and Srikant, R},
  journal={Transactions on Machine Learning Research},
  volume={2024},
  year={2024},
  publisher={Transactions on Machine Learning Research}
}

@inproceedings{
fu2020single,
title={Single-Timescale Actor-Critic Provably Finds Globally Optimal Policy},
author={Zuyue Fu and Zhuoran Yang and Zhaoran Wang},
booktitle={International Conference on Learning Representations},
year={2021}
}

@article{tian2023convergence,
  title={Convergence of actor-critic with multi-layer neural networks},
  author={Tian, Haoxing and Olshevsky, Alex and Paschalidis, Yannis},
  journal={Advances in neural information processing systems},
  volume={36},
  pages={9279--9321},
  year={2023}
}

@inproceedings{blanchet2015unbiased,
  title={Unbiased Monte Carlo for optimization and functions of expectations via multi-level randomization},
  author={Blanchet, Jose H and Glynn, Peter W},
  booktitle={2015 Winter Simulation Conference (WSC)},
  pages={3656--3667},
  year={2015},
  organization={IEEE}
}

@inproceedings{ghosh2023achieving,
  title={Achieving sub-linear regret in infinite horizon average reward constrained mdp with linear function approximation},
  author={Ghosh, Arnob and Zhou, Xingyu and Shroff, Ness},
  booktitle={The Eleventh International Conference on Learning Representations},
  year={2023}
}

@inproceedings{bai2023achieving,
  title={Achieving zero constraint violation for constrained reinforcement learning via conservative natural policy gradient primal-dual algorithm},
  author={Bai, Qinbo and Bedi, Amrit Singh and Aggarwal, Vaneet},
  booktitle={Proceedings of the AAAI Conference on Artificial Intelligence},
  volume={37},
  number={6},
  pages={6737--6744},
  year={2023}
}

@article{jacot2018neural,
  title={Neural tangent kernel: Convergence and generalization in neural networks},
  author={Jacot, Arthur and Gabriel, Franck and Hongler, Cl{\'e}ment},
  journal={Advances in neural information processing systems},
  volume={31},
  year={2018}
}

@inproceedings{jin2021towards,
  title={Towards tight bounds on the sample complexity of average-reward MDPs},
  author={Jin, Yujia and Sidford, Aaron},
  booktitle={International Conference on Machine Learning},
  pages={5055--5064},
  year={2021},
}

@inproceedings{dorfman2022adapting,
  title={Adapting to mixing time in stochastic optimization with markovian data},
  author={Dorfman, Ron and Levy, Kfir Yehuda},
  booktitle={International Conference on Machine Learning},
  pages={5429--5446},
  year={2022},
  organization={PMLR}
}

\clearpage
\appendix
\thispagestyle{empty}

\onecolumn
\aistatstitle{Supplementary Materials}

\section{Pseudocode}
\label{aapendix: pseudocode}

Algorithm~\ref{alg:HiMLMC} combines reward and constraint critic estimation, NPG estimation, and primal-dual updates. The policy and dual variable remain fixed during the estimation subroutines.

As described in Remark~\ref{rem: shared.critic.runs}, Algorithm~\ref{alg:critic} computes all required outputs in one loop of length $h_k$, sharing trajectories across critic copies. Each copy is updated once per iteration with its own stepsize until its prescribed horizon; identical horizons share a copy. Only the base critic runs when the outer correction is inactive. Sharing saves a constant factor in samples, while the two MLMC randomizations yield polylogarithmic expected sample cost.

\begin{algorithm*}[ht]
\caption{HiMLMC-PD-NAC}
\label{alg:HiMLMC}
\begin{algorithmic}[1]

\Require
$\theta_0,\alpha,\beta,\gamma_w,\gamma_\xi(\cdot),
K,H_w,H_0,H_{\max},T_{\max},R,\delta$

\State $\lambda_0\gets 0$
\For{$u\in\{r,c\}$}
    \State Sample $(\zeta_{u,0},b_u)$ as in
    Section~\ref{s: NN.critic}; retain across outer iterations
\EndFor

\For{$k=0,\ldots,K-1$}
    \For{$u\in\{r,c\}$}
        \State $(\eta_{u,k},Q_{u,k})\gets$
        Algorithm~\ref{alg:critic}
        $(\theta_k,u,\zeta_{u,0},b_u,R,
        \gamma_\xi(\cdot),H_0,H_{\max},T_{\max})$
    \EndFor
    \State $w_k\gets$
    Algorithm~\ref{alg:npg}
    $(\theta_k,\lambda_k,\eta_{r,k},\eta_{c,k},
    Q_{r,k},Q_{c,k},\gamma_w,H_w,T_{\max})$
    \State $\theta_{k+1}\gets\theta_k+\alpha w_k$
    \State $\lambda_{k+1}\gets
    \Pi_{[0,2/\delta]}(\lambda_k-\beta\eta_{c,k})$
\EndFor

\State \Return $\{(\theta_k,\lambda_k)\}_{k=0}^{K-1}$

\end{algorithmic}
\end{algorithm*}


\begin{algorithm*}[h]
\caption{Hierarchical MLMC Critic}
\label{alg:critic}
\begin{algorithmic}[1]

\Require
$\theta_k,u,\zeta_{u,0},b_u,R,\gamma_\xi(\cdot),
H_0,H_{\max},T_{\max}$

\State Sample $H_k\sim\mathrm{Geom}(1/2)$ on $\{1,2,\ldots\}$
\State $h_k\gets H_0\left[
1+(2^{H_k}-1)\ones_{\{H_02^{H_k}\leq H_{\max}\}}\right]$
\State $\mathcal H\gets\{H_0\}$;
if $H_02^{H_k}\leq H_{\max},$
set $\mathcal H\gets\{H_0,h_k/2,h_k\}$
\State $\cB_u\gets\{\zeta:\|\zeta-\zeta_{u,0}\|\leq R\}$;
initialize $\xi^{k,H}_{u,0}\gets[0,\zeta_{u,0}^\top]^\top$
for all $H\in\mathcal H$

\For{$h=0,\ldots,h_k-1$}
    \State Independently sample $U_{k,h}\sim\mathrm{Geom}(1/2)$
    and compute $\ell_{k,h}$ using~\eqref{e: traj.length}
    \State Generate $\cT^{k,h}:=(s^{k,h}_t,a^{k,h}_t)_{t=0}^{\ell_{k,h}}$
    under $\pthk$
    \For{$H\in\mathcal H$ with $h<H$}
        \State Construct $\hat{v}^{k,h,H}_{u,\MLMC}$
        using~\eqref{e: ctitic.inner.MLMC}
        at $\xi^{k,H}_{u,h}$ with shared trajectory $\cT^{k,h}$
        \State $\xi^{k,H}_{u,h+1}\gets
        \Pi_{\bR\times\cB_u}\left(
        \xi^{k,H}_{u,h}
        -\gamma_\xi(H)\hat{v}^{k,h,H}_{u,\MLMC}\right)$
    \EndFor
\EndFor

\State Set $\xi^k_{u,H}\gets\xi^{k,H}_{u,H}$
for all $H\in\mathcal H$
\State Compute $Q_{u,k,\MLMC}$ and $\xi_{u,k,\MLMC}$
using~\eqref{e: function .critic.MLMC}
and~\eqref{e: output.critic.MLMC}, evaluating only active corrections
\State \Return $(\eta_{u,k,\MLMC},Q_{u,k,\MLMC})$

\end{algorithmic}
\end{algorithm*}


\begin{algorithm*}[h]
\caption{MLMC NPG Estimator}
\label{alg:npg}
\begin{algorithmic}[1]

\Require
$\theta_k,\lambda_k,\eta_{r,k,\MLMC},\eta_{c,k,\MLMC},
Q_{r,k,\MLMC},Q_{c,k,\MLMC},\gamma_w,H_w,T_{\max}$

\State $w_0^k\gets 0$

\For{$h=0,\ldots,H_w-1$}

    \State Independently sample
    $U_{k,h}\sim\mathrm{Geom}(1/2)$ on $\{1,2,\ldots\}$
    and compute $\ell_{k,h}$ using~\eqref{e: traj.length}

    \State Generate
    $\cT^{k,h}:=(s^{k,h}_t,a^{k,h}_t)_{t=0}^{\ell_{k,h}}$
    under $\pthk$

    \State For $t=0,\ldots,\ell_{k,h}-1,$ independently draw
    $\bar a_t^{k,h}\sim\pthk(\cdot|s_t^{k,h})$
    conditional on the trajectory

    \State Construct the MLMC Fisher estimate
    $\hF_{\MLMC}$ and PG estimates
    $\hg_{r,\MLMC},\hg_{c,\MLMC}$
    using~\eqref{e: naive.NPG.estimator}
    and~\eqref{e: NPG.MLMC.estimator}

    \State $w_{h+1}^k\gets w_h^k
    -\gamma_w\left[
    \hF_{\MLMC}w_h^k
    -\big(\hg_{r,\MLMC}+\lambda_k\hg_{c,\MLMC}\big)
    \right]$

\EndFor

\State \Return $w_{H_w}^k$

\end{algorithmic}
\end{algorithm*}


\section{Bounding Optimality Gap: Proof of Lemma~\ref{lem: optimality.gap.bound}}
\label{appendix: opt.gap}

    Using the fact that $\lambda_0=0,$ we have for all $K>0,$ 
    \begin{align*}
        \lambda^2_K 
        & = \sum^{K-1}_{k=0} \left( \lambda^2_{k+1} - \lambda^2_k\right) 
        \\
        & \overset{(a)}{=} \sum^{K-1}_{k=0} \left(\Pi_{[0,2/\delta]}[\lambda_k -\beta  \eta_{c,k,\MLMC}  ]\right)^2- \left(\Pi_{[0,2/\delta]}(\lambda_k)\right)^2 
        \\
        & \overset{(b)}{\leq} \sum^{K-1}_{k=0} \left(\lambda_k -\beta \eta_{c,k,\MLMC}\right)^2  - \lambda^2_k 
        \\
        & \leq \beta^2 \sum^{K-1}_{k=0} \eta^2_{c,k,\MLMC} -2\beta \sum^{K-1}_{k=0} \lambda_k\eta_{c,k,\MLMC} 
        \\
        & \overset{(c)}{\leq} 2\beta^2 \sum^{K-1}_{k=0} \left|\eta_{c,k,\MLMC} - \Jthk_c\right|^2 + 2\beta^2\sum^{K-1}_{k=0}\big(\Jthk_c\big)^2
        + 2\beta \sum^{K-1}_{k=0} \lambda_k\left( \Jthk_c - \eta_{c,k,\MLMC} \right) 
        + 2\beta \sum^{K-1}_{k=0} \lambda_k\left( J^{\piS}_c- \Jthk_c \right),
    \end{align*}
    where $(a)$ uses from the dual update in~\eqref{e: actor.update}, $(b)$ follows since any projection is non-expansive, and $(c)$ follows since $J^{\piS}_c\geq 0,$ and the fact that $c(s,a)\leq 1,$ for all $s,a$ implies $\eta_{c,k}\leq 1.$
    Since $\lambda_K\geq 0,$ dividing both sides by $2\beta K$ and taking conditional expectation  gives us
    \begin{align}\label{e: contraint.lower}
        -\frac{1}{K}\sum^{K-1}_{k=0} \lambda_k\big[ J^{\piS}_c - \Jthk_c \big]
        & \leq \beta + \frac{1}{K}\sum^{K-1}_{k=0} \lambda_k\left( \Jthk_c - \bE_k\big[\eta_{c,k,\MLMC}\big] \right)
        + \frac{\beta}{2K}\sum_{k=0}^{K-1}\bE_k\left|\Jthk_c - \eta_{c,k,\MLMC} \right|^2
        \nonumber\\
        & \overset{(a)}{\leq} \beta + \frac{2}{\delta K}\sum^{K-1}_{k=0} \left|\Jthk_c - \bE_k\big[\eta_{c,k,\MLMC}\big] \right|
        + \frac{\beta}{2K}\sum_{k=0}^{K-1}\bE_k\left|\Jthk_c - \eta_{c,k,\MLMC} \right|^2
        \nonumber\\
        & \overset{(b)}{\leq} \beta + \frac{2}{\delta K}\sum^{K-1}_{k=0} \left\|\xi^*_{c,k} - \bE_k\big[\xi_{c,k,\MLMC}\big] \right\|
        + \frac{\beta}{2K}\sum_{k=0}^{K-1}\bE_k\left\|\xi^*_{c,k} - \xi_{c,k,\MLMC} \right\|^2,
    \end{align}
    where $(a)$ follows since $|\Jthk_c|\leq 1,$ $0\leq\lambda_k\leq 2/\delta,$ and $(b)$ follows from the definitions of $\xi^*_{c,k}$ and $\xi_{c,k,\MLMC}.$ 
    
    Using the definition of $\cL(\theta,\lambda)$ and $\cL(\piS,\lambda),$ we have
    \begin{equation}
        \frac{1}{K}\sum_{k=0}^{K-1}\bE\big[J^{\piS}_r -\Jthk_r \big] = -\frac{1}{K}\sum^{K-1}_{k=0} \bE\big[\lambda_k\big( J^{\piS}_c - \Jthk_c \big)\big]
        + \frac{1}{K}\sum_{k=0}^{K-1}\bE\big[\cL(\piS, \lambda_k) - \cL(\theta_k,\lambda_k) \big].
    \end{equation}
    We combine the above relation with~\eqref{e: contraint.lower} to complete the proof of Lemma~\ref{lem: optimality.gap.bound}.
\hfill\qed

\section{Bounding Constraint Violation Rate: Proof of Lemma~\ref{lem: constraint.violation.rate}}
\label{appendix: constraint.violation}
    
    We have from the dual update in~\eqref{e: actor.update}, that
    \begin{align}
        & \bE_k\left(\lambda_{k+1} - \frac{2}{\delta} \right)^2 - \left(\lambda_k - \frac{2}{\delta} \right)^2 
        \nonumber\\
        & \overset{(a)}{\leq}   \bE_k\left(\lambda_k - \beta\ \eta_{u,k,\MLMC} - \frac{2}{\delta} \right)^2 - \left(\lambda_k - \frac{2}{\delta} \right)^2 
        \nonumber\\
        & =  \beta^2 \bE_k\eta_{c,k,\MLMC}^2 - 2\beta \bE_k \left[ \eta_{c,k,\MLMC}\left( \lambda_k-\frac{2}{\delta} \right)\right]
        \nonumber\\
        & \leq  2\beta^2 \bE_k\left| \Jthk_c - \eta_{c,k,\MLMC} \right|^2 
        + 2\beta\left( \lambda_k-\frac{2}{\delta} \right)\left( \Jthk_c - \bE_k\big[\eta_{c,k,\MLMC}\big] \right) 
        - 2\beta \Jthk_c\left( \lambda_k-\frac{2}{\delta} \right) + 2\beta^2 (\Jthk_c)^2  
        \nonumber\\
        & \overset{(b)}{=} \cO\bigg( \beta^2\bE_k\left\| \xi^*_{c,k} - \xi_{c,k,\MLMC} \right\|^2 
        + \beta \left\| \xi^*_{c,k} - \bE_k\big[\xi_{c,k,\MLMC} \big]\right\|   - \beta \Jthk_c\left( \lambda_k-\frac{2}{\delta} \right) + \beta^2
        \bigg),
    \end{align}
    where $(a)$ follows from non-expansiveness of the projection $\Pi_{[0,2/\delta]},$ and $(b)$ follows by using the definitions of $\xi^*_{c,k}$ and $\xi_{c,k,\MLMC},$ and using $\lambda_k\leq 2/\delta$ and $\Jthk_c\leq 1.$ 

    Summing both sides over $0\leq k<K,$  multiplying with $\beta/K,$ and taking expectation gives us
    \begin{align}\label{e: const.upper}
        &\frac{1}{K} \sum_{k=0}^{K-1}\bE\left[\Jthk_c\left( \lambda_k-\frac{2}{\delta} \right) \right]
        \nonumber\\
        &= \cO\Bigg( \beta + \frac{1}{\beta K}+  \frac{\beta}{K}\sum_{k=0}^{K-1}\bE\left\| \xi^*_{c,k} - \xi_{c,k,\MLMC} \right\|^2 
        +  \frac{1}{K}\sum_{k=0}^{K-1}\bE\left\| \xi^*_{c,k} - \bE_k\big[\xi_{c,k,\MLMC} \big]\right\| \Bigg).
    \end{align}
Now, note that 
\begin{align}
    \sum_{k=0}^{K-1}\left[J^{\piS}_r - \Jthk_r \right] - \frac{2}{\delta}\sum_{k=0}^{K-1}\Jthk_c 
    & \overset{(a)}{\leq} \sum_{k=0}^{K-1}\left[J^{\piS}_r - \Jthk_r \right] +  \sum_{k=0}^{K-1}\left[J^{\piS}_c - \frac{2}{\delta}\Jthk_c \right]
    \nonumber\\
    & \overset{(b)}{=} \sum_{k=0}^{K-1}\left[\cL(\piS,\lambda_k) - \cL(\theta_k,\lambda_k) \right] + \sum_{k=0}^{K-1}\left( \lambda_k-\frac{2}{\delta}\right)\Jthk_c.
\end{align}
where $(a)$ uses the fact that  $J^{\piS}_c\geq 0,$ while $(b)$ uses the definition of the Lagrangian $\cL.$ Dividing both sides by $K,$ taking expectation, and substituting the last term using~\eqref{e: const.upper} then gives us
\begin{multline}\label{e: intermediate}
    J^{\piS}_r  - \frac{1}{K}\sum_{k=0}^{K-1}\bE\left[\Jthk_r \right] + \frac{2}{\delta K}\sum_{k=0}^{K-1}\bE\left[-\Jthk_c \right]
    = \cO\Bigg( \beta + \frac{1}{\beta K} + \frac{1}{K}\sum_{k=0}^{K-1}\big[\cL(\piS,\lambda_k) - \cL(\theta_k,\lambda_k) \big] 
    \\
    + \frac{\beta}{K}\sum_{k=0}^{K-1}\bE\left\| \xi^*_{c,k} - \xi_{c,k,\MLMC} \right\|^2  +  \frac{1}{K}\sum_{k=0}^{K-1}\bE\left\| \xi^*_{c,k} - \bE_k\big[\xi_{c,k,\MLMC} \big]\right\| 
    \Bigg).
\end{multline}
We define a policy $\bar{\pi}_K$ that picks a policy from  $\{\pi_{\theta_k}\}_{k<K},$ uniformly at random. Then,~\eqref{e: intermediate} implies
\begin{multline}
    J^{\piS}_r - \bE\big[ J^{\bar{\pi}_K}_r \big] + \frac{2}{\delta}\bE\big[- J^{\bar{\pi}_K}_c \big]
    = \cO\Bigg( \beta + \frac{1}{\beta K} +  \frac{1}{K}\sum_{k=0}^{K-1}\big[\cL(\piS,\lambda_k) - \cL(\theta_k,\lambda_k) \big]
    \\
    + \frac{\beta}{K}\sum_{k=0}^{K-1}\bE\left\| \xi^*_{c,k} - \xi_{c,k,\MLMC} \right\|^2  + \frac{1}{K}\sum_{k=0}^{K-1}\bE\left\| \xi^*_{c,k} - \bE_k\big[\xi_{c,k,\MLMC} \big]\right\|\Bigg).
\end{multline}
By \cite[Lemma G.6]{xu2026global}, we conclude
\begin{align}
    \bE\big[- J^{\bar{\pi}_K}_c \big] = \frac{1}{K}\sum_{k=0}^{K-1}\bE\big[- \Jthk_c \big]
    &= \cO\Bigg( \frac{\delta}{\beta K}+ \delta\beta + \frac{\delta}{K}\sum_{k=0}^{K-1}\big[\cL(\piS,\lambda_k) - \cL(\theta_k,\lambda_k) \big] 
    \nonumber\\
    & \quad + \frac{\delta\beta}{K}\sum_{k=0}^{K-1}\bE\left\| \xi^*_{c,k} - \xi_{c,k,\MLMC} \right\|^2  +  \frac{\delta}{K}\sum_{k=0}^{K-1}\bE\left\| \xi^*_{c,k} - \bE_k\big[\xi_{c,k,\MLMC} \big]\right\|
    \Bigg).
\end{align}
This completes the proof of Lemma~\ref{lem: constraint.violation.rate}.
\hfill\qed

\section{Lagrange Error Analysis: Proof of Lemma~\ref{lem: Lagrange.error}}
\label{appendix: lagrange}
    Let expectaions w.r.t stationary distributions $d^{\piS}$ or $\nu^{\piS}$ be denoted by $\bE_{*}[\cdot].$ Then, we have for for $0\leq k<K,$
    \begin{align}
        &\bE_{*}\left[ \KL\left(\piS(\cdot| s)||\pi_{\theta_k}(\cdot|s)\right) - \KL\left(\piS(\cdot|s)|| \pi_{\theta_{k+1}}(\cdot|s)\right) \right] 
        \nonumber\\
        & = \bE_{*}\left[\ln\frac{\pi_{\theta_{k+1}}(a|s)}{\pi_{\theta_k}(a|s)} \right]
        \nonumber\\
        & \overset{(a)}{\geq} \bE_{*}\left\langle\nabla_\theta \ln\pi_{\theta_k}(a|s), \theta_{k+1} -\theta_k \right\rangle 
        - \frac{G_2}{2}\|\theta_{k+1} - \theta_k\|^2
        \nonumber\\
        & \overset{(b)}{=} \alpha\bE_{*}\left\langle\nabla_\theta \ln\pi_{\theta_k}(a|s), w_k \right\rangle  - \frac{G_2\alpha^2}{2}\|w_k\|^2 
        \nonumber\\
        & \geq \alpha\bE_{*}\left\langle\nabla_\theta \ln\pi_{\theta_k}(a|s), w_k \right\rangle - \frac{G_2\alpha^2}{2}\|w_k\|^2  
        \nonumber\\
        & \overset{(c)}{=} \alpha\left[ \cL(\piS,\lambda_k) - \cL(\theta_k,\lambda_k) \right]  
         - \frac{G_2\alpha^2}{2}\|w_k\|^2   
        + \alpha\bE_{*}\left\langle\nabla_\theta \ln\pi_{\theta_k}(a|s), w_k \right\rangle 
        -\alpha\left[ J^{\piS}_r -\Jthk_r \right] - \alpha\lambda_k\left[ J^{\piS}_c -\Jthk_c \right] 
        \nonumber\\
        & \overset{(d)}{=} \alpha\left[ \cL(\piS,\lambda_k) - \cL(\theta_k,\lambda_k) \right] 
        - \frac{G_2\alpha^2}{2}\|w_k\|^2 + \alpha\bE_{*}\left\langle\nabla_\theta \ln\pi_{\theta_k}(a|s), w_k - \wS_k \right\rangle
        \nonumber\\
        & \quad + \alpha\bE_{*}\left[\left\langle\nabla_\theta \ln\pi_{\theta_k}(a|s), \wS_k \right\rangle -  A_r^{\theta_k}(s,a)  - \lambda_kA_c^{\theta_k}(s,a) \right]
        \nonumber\\
        & \overset{(e)}{\geq} \alpha\left[ \cL(\piS,\lambda_k) - \cL(\theta_k,\lambda_k) \right] 
        - \frac{G_2\alpha^2}{2}\|w_k\|^2 - \alpha\sqrt{\epsb}+ \alpha\bE_{*}\left\langle\nabla_\theta \ln\pi_{\theta_k}(a|s), w_k - \wS_k \right\rangle,
    \end{align}
    where $(a)$ follows from Assumption~\ref{a: score.function}, $(b)$ follows from~\eqref{e: actor.update}, $(c)$ follows from~\eqref{e: lagrange.function}, $(d)$ follows from Lemma~\ref{app.lem: performance.diff}, and $(e)$ follows from combining Assumption~\ref{a: FA.error.PG} with the fact that $x\mapsto\sqrt{x}$ is concave and Assumption~\ref{a: FA.error.PG}.

    Now, rearranging terms and taking the conditional expectation w.r.t $\cF_k$  gives
    \begin{align}
        & \bE_k\left[\cL(\piS, \lambda_k) - \cL(\theta_k,\lambda_k) \right] - \sqrt{\epsb} - \frac{1}{\alpha}
        \bE_{*}\left[ \KL\left(\piS(\cdot| s)||\pi_{\theta_k}(\cdot|s)\right) - \KL\left(\piS(\cdot|s)|| \pi_{\theta_{k+1}}(\cdot|s)\right) \right]
        \nonumber\\
        & \overset{(a)}{\leq} -\left\langle \bE_{\piS}\left[\nabla_\theta \ln \pi_{\theta_k}(a|s)\right], \bE_k[w_k] - \wS_k\right\rangle + \frac{G_2\alpha}{2}\|w_k\|^2 
        \nonumber\\
        & \overset{(b)}{\leq}   G_1\|\bE_k[w_k] - \wS_k\| + G_2\alpha\left(\bE\|w_k-\wS_k\|^2 + \bE\|\wS_k\|^2\right)
        \nonumber\\
        & \overset{(c)}{\leq}   G_1\|\bE_k[w_k] - \wS_k\| + G_2\alpha\bE\|w_k-\wS_k\|^2 + \cO(\tmix \alpha),
    \end{align}
    where $(a)$ follows since $\theta_k$ is $\cF_k$-measurable and $w_k$ is independent of $\nu^{\piS},$ $(b)$ uses the inequality $(a+b^2)\leq 2a^2+2b^2,$ while $(c)$ follows as we use Assumptions~\ref{a: slater},~\ref{a: score.function}, and~\ref{a: Fisher.degenrate} to get 
    \begin{equation}
        \big\|\wS_k\big\|=\big\|F(\theta_k)^\dagger\nabla_\theta\cL(\theta_k,\lambda_k)\big\|
        \leq \frac{1}{\mu_F}\left(\big\|\nabla_\theta \Jthk_r\big\| + \frac{2}{\delta}\big\|\nabla_\theta \Jthk_c \big\|\right),
    \end{equation}
    and use Assumption~\ref{a: ergodicity} to conclude 
    \begin{equation}
        \big\|\nabla_\theta \Jthk_u\big\| \leq G_1\bE_{\piS}\big\|A^{\theta_k}_u(s,a)\big\|= \cO(\tmix).
    \end{equation}
    Finally, the result follows by summing both sides over $0\leq k < K$, dividing by $ K$, using the fact that the KL-divergence is non-negative, and taking the expectation.
\hfill\qed

\section{NPG Estimator Analysis: Proof of Theorem~\ref{thm: NPG.bounds}}
\label{appendix: NPG.analysis}

Throughout subsequent sections, we write $\bE_{k,h}[\cdot]:=\bE[\cdot\mid\cF_{k,h}],$
where $\cF_{k,h}$ is generated by $\mathcal I_0$ and all algorithmic randomness preceding inner iteration $h$ of
outer iteration $k$.

Recall from~\eqref{e: NPG.update}, that the NPG update is given by the following linear stochastic recursion:
\begin{equation}\label{e: dummy.NPG.update}
    w^k_{h+1} = w^k_{h} 
    + \gamma_w\left[ \left(\hg_{r,\MLMC} + \lambda_k \hg_{c,\MLMC}\right) - \hF_\MLMC\ w^k_h \right],
\end{equation}
where $\big(\hg_{r,\MLMC} + \lambda_k \hg_{c,\MLMC}\big)$ and $ \hF_\MLMC$ are MLMC-based estimators for the Lagrangian gradient $\nabla_\theta \cL(\theta_k, \lambda_k)$ and the Fisher matrix $F(\theta_k),$ respectively. 

Note that, under Assumption~\ref{a: Fisher.degenrate}, $\wS_k$ is the unique vector satisfying $F(\theta_k)\wS_k = \nabla_\theta \cL(\theta_k, \lambda_k).$ We claim that the NPG update satisfies the conditions needed in Lemma~\ref{app.lem: gen.linear}, which bounds the convergence rate achieved by a general linear stochastic recursion. 

Note that, by Assumptions~\ref{a: score.function} and~\ref{a: Fisher.degenrate}, for all $ \theta\in \Theta$ and $\lambda\in [0,2/\delta],$ 
\begin{equation}
    \mu_F\leq\|F(\theta)\|\leq G^2_1
    \qquad
    \text{and} \qquad \|\nabla_\theta \cL(\theta, \lambda)\|\leq \left(1+\frac{2}{\delta} \right)4G_1\tmix.
\end{equation}
Next, we bound the bias and MSE of MLMC estimators $\hg_{u,\MLMC}$ and $ \hF_\MLMC$. Equation~\eqref{e: naive.NPG.estimator} shows the naive estimator for $F(\theta_k)$ as $\frac{1}{T}\sum_{t=0}^{T-1}\hat{F}^k(s^{kh}_t,a^{kh}_t)=\nabla_\theta \ln\pi_{\theta_k}(a|s)\left(\nabla_\theta \ln\pi_{\theta_k}(a|s)\right)^\top.$ From Assumption~\ref{a: score.function}, we have
\begin{equation}
    \left\|\bE_{\theta_k}\left[\hat{F}^{kh}(s, a)\right] - F(\theta_k)\right\| = 0 
    \qquad
    \text{and} \qquad \left\|\hat{F}^{kh}(s^{kh}_t, a^{kh}_t) - F(\theta_k) \right\|^2 \leq 4G^4_1,
\end{equation}
Thus, by Lemma~\ref{app.lem: MLMC}, we we have the bias and MSE of $\hat{F}^{kh}_\MLMC$ as 
\begin{equation}\label{e: bounds.fisher}
    \begin{aligned}
        & \delta^2_F:=  \left\|\bE_{kh}\left[\hat{F}^{kh}_\MLMC\right] - F(\theta_k)\right\|^2 = \cO\left( \tmix/\Tm \right) ,
        \\
        &       \sigma^2_F:=\bE_{kh}\left\|\hat{F}^{kh}_\MLMC - F(\theta_k) \right\|^2 =  \cO\left( \tmix\ln\Tm \right).
    \end{aligned}
\end{equation}
Likewise, to we bound the MSE and bias in $\hg_{u,\MLMC},$ recall from~\eqref{e: naive.NPG.estimator} that the naive estimator for $\nabla_\theta \Jthk_u$ is given by $\frac{1}{T}\sum_{t=0}^{T-1}\hg_u(s^{k,h}_t, a^{k,h}_t, \bar{a}^{k,h}_t, s^{k,h}_{t+1}, a^{k,h}_{t+1}),$ where
\begin{equation}
    \hg_u(s,a, \bar{a},s',a') =  \Big( u(s,a) - \eta_{u,k}
    + Q_{u,k,\MLMC}(s',a') - Q_{u,k,\MLMC}(s,\bar{a})\Big)\nabla_\theta \ln \pi_\theta(a|s).
\end{equation}
Here, $Q_{u,k,\MLMC}$ denotes the critic estimator for $\pthk$ defined in~\eqref{e: output.critic.MLMC}. 

Recall from~\eqref{e: linear.approx} that $\Ql$ be the first-order Taylor approximation to $Q(\cdot,\cdot, \zeta).$
Then, we have
\begin{align}
    & \left\| \hg_u(s,a, \bar{a},s',a') \right\|  
    \nonumber\\
    & \leq  G_1\Big( \big|u(s,a)\big| +\big|\eta_{u,k,\MLMC}\big| + 2\big\|Q_{u,k,\MLMC}\big\|  \Big)
    \nonumber\\
    & = \cO\bigg( \big|\eta_{u,k,\MLMC}\big| +  \big\|\Ql(\cdot,\cdot, \zeta_{u,k,\MLMC})\big\| + \big\| Q_{u,k,\MLMC} - \Ql(\cdot,\cdot, \zeta_{u,k,\MLMC})\big\| \bigg)
    \nonumber\\
    & = \cO\bigg( \big|\eta_{u,k,\MLMC}\big| +  \sqrt{\ln(\Hm/\varepsilon)}\big\|\zeta_{u,k,\MLMC}\big\| + \big\| Q_{u,k,\MLMC} - \Ql(\cdot,\cdot, \zeta_{u,k,\MLMC})\big\| \bigg).
\end{align}
Thus, the single-sample MSE for the naive PG estimator is
\begin{equation}
    \bE_k\left\| \hg_u(s^{k,h}_t,a^{k,h}_t, \bar{a}^{k,h}_t,s^{k,h}_{t+1},a^{k,h}_{t+1}) - \nabla_\theta \Jthk_u \right\|^2 
    = \widetilde{\cO}\left( \sigma^2_Q + \ln\left(\frac{\Hm}{\varepsilon}\right)\left\|\xi_{u,k,\MLMC} - \xi^*_{u,k}\right\|^2 \right),
\end{equation}
where we define 
\begin{equation*}
    \sigma^2_{Q}:= \bE_k\big\| Q_{u,k,\MLMC} - \Ql(\cdot,\cdot, \zeta_{u,k,\MLMC})\big\|^2.
\end{equation*}
Now, we bound the single-sample bias for the naive PG estimator. Note that,
\begin{align}\label{e: NPG.bias.decomp}
    &\bE_{\theta_k}\bigg[ \hg_u(s,a, \bar{a},s',a') \bigg] 
    \nonumber\\
    & = \bE_{\theta_k}\bigg[\nabla_\theta \ln\pthk(s,a)
    \Big[\left(u(s,a) - \Jthk_u  + \Qthk_u(s',a') - \Qthk_u(s,\bar{a})\right)
    \nonumber\\
    & + \left( \Jthk_u - \eta_{u,k,\MLMC}\right)  
    + \big( \Qthk(s,\bar{a}) - \Ql(s,\bar{a},\zeta^*_{u,k})\big)
    - \left(\Qthk(s',a') - \Ql(s',a',\zeta^*_{u,k})\right)
    \nonumber\\
    & + \left( \Ql(s',a', \zeta_{u,k,\MLMC}) -  \Ql(s',a', \zeta^*_{u,k}) \right)
    -  \left( \Ql(s,\bar{a}, \zeta_{u,k,\MLMC}) -  \Ql(s,\bar{a}, \zeta^*_{u,k}) \right) 
    \nonumber\\
    & + \big(Q_{u,k,\MLMC}(s',a') -\Ql(s',a', \zeta_{u,k,\MLMC})\big)
    - \big(Q_{u,k,\MLMC}(s',a') -\Ql(s',a', \zeta_{u,k,\MLMC}) \big)
    \Big] \bigg].
\end{align}

Restricting to the NTK event $\Entk,$ applying the NTK bounds from Lemma~\ref{app.lem: NTK.bounds} on~\eqref{e: NPG.bias.decomp} and using~\eqref{e: policy.gradient}, we have the single-sample PG bias as
\begin{align}
    &\left\|\bE_{\theta_k}\bigg[ \hg_u(s,a,\bar{a}, s',a') \bigg] - \nabla_\theta \Jthk  \right\|^2
    \nonumber\\
    & = \widetilde{\cO}\bigg( \epsap + \frac{\ln(\Hm/\varepsilon)}{m} + \ln(\Hm/\varepsilon)\left\|\xi_{u,k,\MLMC} - \xi^*_{u,k} \right\|^2
    + \big\|Q_{u,k,\MLMC} -\Ql(\cdot,\cdot, \zeta_{u,k,\MLMC})\big\|^2  \bigg)
\end{align}
and
\begin{align}
    &\left\|\bE_k\bE_{\theta_k}\bigg[ \hg_u(s,a,\bar{a},s',a') \bigg] - \nabla_\theta \Jthk  \right\|^2
    \nonumber\\
    & = \widetilde{\cO}\bigg( \epsap +  \frac{\ln(\frac{\Hm}{\varepsilon})}{m} + \ln\left(\frac{\Hm}{\varepsilon}\right)\left\|\bE_k\big[\xi_{u,k,\MLMC}\big] - \xi^*_{u,k} \right\|^2
    + \left\|\bE_k\big[Q_{u,k,\MLMC} -\Ql(\cdot,\cdot, \zeta_{u,k,\MLMC})\big]\right\|^2  \bigg),
\end{align}
where we adapt the arguments in~\citet{ke2024improved} to average-reward to bound $\bE_{k,h}\|\Qthk_u-\Ql(\cdot,\cdot,\zeta^*_{u,k})|=\cO(\epsap).$ 
Consequently, applying Lemma~\ref{app.lem: MLMC}, we get the following MSE and bias for $\hg_{u,\MLMC}:$
\begin{align}\label{e: MSE.PG}
    \sigma^2_g &:= \sum_u\bE_{k,h}\left\| \hg_{u,\MLMC} - \nabla_\theta \Jthk_u \right\|^2 
    \\
    & = \widetilde{\cO}\bigg( \epsap + \frac{\ln(\Hm/\varepsilon)}{m} + \sigma^2_Q + \ln\left(\frac{\Hm}{\varepsilon}\right)\sum_{u}\left\|\xi_{u,k,\MLMC} - \xi^*_{u,k} \right\|^2 \bigg), 
\end{align}
and
\begin{equation}\label{e: bias.PG}
    \begin{aligned}
        \delta^2_g & := \sum_{u}\left\| \bE_{k,h}\big[\hg_{u,\MLMC}\big] - \nabla_\theta \Jthk_u \right\|^2 
        \nonumber\\
        & = \widetilde{\cO}\bigg( \epsap + \frac{\ln(\Hm/\varepsilon)}{m} + 
        \sigma^2_Q + \ln\left(\frac{\Hm}{\varepsilon}\right)\sum_u\left\|\xi_{u,k,\MLMC} - \xi^*_{u,k} \right\|^2 \bigg),
        \\
        \bar{\delta}^2_g &:= \sum_{u}\left\| \bE_k\big[\hg_{u,\MLMC}\big] - \nabla_\theta \Jthk_u \right\|^2
        \nonumber\\
        & =\widetilde{\cO}\bigg( \epsap + \frac{\ln(\Hm/\varepsilon)}{m} + \delta^2_Q + \ln\left(\frac{\Hm}{\varepsilon}\right)\sum_u\left\|\bE_k\big[\xi_{u,k,\MLMC}\big] - \xi^*_{u,k} \right\|^2 \bigg),
    \end{aligned}
\end{equation}
where the sum is over $u\in \{r,c \}$ and
\[
    \delta^2_Q:= \left\| \bE_k\big[ Q_{u,k,\MLMC} - \Ql(\cdot,\cdot, \zeta_{u,k,\MLMC})\big]\right\|^2.
\]
Using these bounds, we are ready to apply Lemma~\ref{app.lem: gen.linear} to the NPG update~\eqref{e: dummy.NPG.update}. Choosing the NPG stepsize $\gamma_w:= \frac{2\ln T}{\mu_F \bH},$ we have for $T$ sufficiently large:
\begin{align*}
    & \bE_k\big\|w_k - \wS_k\big\|^2 
    = \widetilde{\cO}\bigg( \frac{1}{T^2} + \frac{1}{\bH}\left(\sigma^2_F + \sigma^2_g \right) + \delta^2_F + \delta^2_g \bigg),
    \\
    &\big\| \bE_k[w_k] - \wS_k \big\|^2 
    =\widetilde{\cO}\bigg( \frac{1}{T^2}   +  \delta^2_F  +  \bar{\delta}^2_g 
    + \delta^2_P\left( \frac{1}{\bH}\left[ \sigma^2_F + \sigma^2_g \right]  +  \delta^2_F + \delta^2_g \right)\bigg).
\end{align*}
Finally, substituting the bias and MSE terms in the above relations with~\eqref{e: bounds.fisher},~\eqref{e: MSE.PG}, and~\eqref{e: bias.PG} completes the proof of Theorem~\ref{thm: NPG.bounds}.    
\hfill\qed

\section{Critic Estimator Analysis: Proofs of Lemma~\ref{lem: vanilla.critic.bound} and Theorem~\ref{thm: critic.bounds.MLMC}}
\label{appendix: critic.analysis}

\begin{proof}[\textbf{Proof of Lemma~\ref{lem: vanilla.critic.bound}}]
We begin by writing the vanilla-critic update as a perturbed linear recursion. Let $z,z'$ denote the state action pairs $(s,a),(s',a'),$ respectively. Further, let $\psi(z)$ denote $\nabla_\zeta Q(z,\zeta_0).$ Then, replacing the neural critic $Q(\cdot,\cdot,\zeta)$ with the linear critic $\Ql(\cdot,\cdot,\zeta)$ in~\eqref{e: critic.gradient}, we decompose the naive single-sample critic-gradient estimator as 
\begin{align}\label{e: critic.gradient.decomp}
    \hv(z,z') 
    := \left[\hA(z,z')\ \xi^k_{u,h} - \hB_u(z,z')\right] 
    - \hV_u(z,z'),
\end{align}
where 
\begin{align}
    \hA(z,z'):= \begin{bmatrix} c_\gamma & 0 \\ \psi_z & \psi(z)\left[ \psi(z)- \psi(z')\right]^\top \end{bmatrix}
    \quad \text{and} \quad
    \hB_u(z,z'):= \begin{bmatrix} c_\gamma u(z)  \\ \left( u(z)-Q(z,\zeta_0) + Q(z',\zeta_0)\right)\psi(z)  \end{bmatrix},
\end{align}
and $\hV_u(z,z')$ collects the remainder terms. This decomposition then translates to the following decomposition of the MLMC critic-gradient estimator in~\eqref{e: ctitic.inner.MLMC}:
\begin{equation}\label{e: critic.MLMC.decomp}
    \hv_\MLMC 
    := \left[\hA_\MLMC\ \xi^k_{u,h} - \hB_{u,\MLMC}\right] 
    - \hV_{u,\MLMC},
\end{equation}
where $\hA_\MLMC, \hB_{u,\MLMC},$ and $\hV_{u,\MLMC}$ are obtained by replacing $\hv_u$ in~\eqref{e: ctitic.inner.MLMC} with $\hA,\hB_u,$ and $\hV_u,$ respectively.

With the decomposition~\eqref{e: critic.MLMC.decomp}, the vanilla-critic update~\eqref{e: critic.update} can be written as the following perturbed linear recursion:
\begin{equation}\label{e: critic.update.decomposed}
    \xi^k_{u,h+1} 
    = \Pi_\cB\Big( \xi^k_{u,h} 
    + \gamma_{\xi,H}\big[\hB_{u,\MLMC} - \hA_{\MLMC}\ \xi^k_{u,h}\big]
    + \gamma_{\xi,H}\hV_{u,\MLMC} \Big).
\end{equation}
Recall that $\xi^*_{u,k}$ is a solution to   $ A(\theta_k)\xi = B_u(\theta_k),$ where 
\begin{equation}\label{e: linear.critic.operator}
        A(\theta_k) := \bE_{\theta_k}[\hA(z,z')]
        \qquad 
        \text{and} \qquad B_u(\theta_k) := \bE_{\theta_k}[\hB_u(z,z')].
\end{equation}
Let $\hA_\MLMC,$ $\hB_{u,\MLMC},$ and $\hV_{u,\MLMC}$ satisfy the following MSE bounds
\begin{equation}\label{e: linear.critic.MSE.bounds}
    \bE_{k,h}\|\hA_\MLMC - A(\theta_k)\|^2 \leq \sigma^2_A, 
    \qquad
    \bE_{k,h}\|\hB_{u,\MLMC} - B_u(\theta_k)\|^2 \leq \sigma^2_B,
    \quad
    \text{and} \qquad  \|\hV_{u,\MLMC}\|^2 \leq \sigma^2_V,
\end{equation}
and the following bias bounds
\begin{equation}\label{e: linear.critic.bias.bounds}
    \|\bE_{k,h}[\hA_\MLMC] - A(\theta_k)\|^2 \leq \delta^2_A
    \qquad
    \text{and} \qquad \|\bE_{k,h}[\hB_{u,\MLMC}] - B_u(\theta_k)\|^2 \leq \delta^2_B.
\end{equation}
If we define the auxiliary linear recursion
\begin{equation}\label{e: linear.critic.update.decomposed}
    \chi^k_{u,h+1} 
    = \Pi_\cB\Big( \chi^k_{u,h} 
    + \gamma_{\xi,H}\big[\hB_{u,\MLMC} - \hA_{\MLMC}\ \chi^k_{u,h}\big]
     \Big),
\end{equation}
where $\chi^k_{u,0} = \xi^k_{u,0},$ then subtracting~\eqref{e: linear.critic.update.decomposed} from~\eqref{e: critic.update.decomposed} gives us 
\begin{align}
    & \left\|\xi^k_{u,h+1} - \chi^k_{u,h+1}\right\|^2 
    \nonumber\\
    & \overset{(a)}{\leq} \Big\|\big(I - \gamma_{\xi,H}\hA_\MLMC \big)\left(\xi^k_{u,h} - \chi^k_{u,h}\right)  
    + \gamma_{\xi,H}\hV_{u,\MLMC} \Big\|^2
    \nonumber\\
    & = \Big\|\big(I - \gamma_{\xi,H}A(\theta_k) \big)\left(\xi^k_{u,h} - \chi^k_{u,h}\right)  
    + \gamma_{\xi,H}\left( A(\theta_k)- \hA_\MLMC \right)\left(\xi^k_{u,h} - \chi^k_{u,h}\right) + \gamma_{\xi,H}\hV_{u,\MLMC} \Big\|^2
    \nonumber\\
    &
    = \Big\| \big(I - \gamma_{\xi,H}A(\theta_k) \big)\left(\xi^k_{u,h} - \chi^k_{u,h}\right)\Big\|^2 
    + \gamma^2_{\xi,H}\left\| A(\theta_k)- \hA_\MLMC \right\|^2\left\|\xi^k_{u,h} - \chi^k_{u,h}\right\|^2 + \gamma^2_{\xi,H}\left\| \hV_{u,\MLMC} \right\|^2
    \nonumber\\
    & \quad + 2\gamma^2_{\xi,H}\left\langle A(\theta_k) - \hA_\MLMC, \hV_{u,\MLMC} \right\rangle
    + 2\gamma_{\xi,H}\Big\langle A(\theta_k) - \hA_\MLMC,
    \big(I - \gamma_{\xi,H}A(\theta_k) \big)\left(\xi^k_{u,h} - \chi^k_{u,h}\right) \Big\rangle,
\end{align}
where $(a)$ follows from the non-expansiveness of $\Pi_\cB.$ Taking conditional expectation on both sides, we have
\begin{align}\label{e: MSE.diff.recursive}
    & \bE_{k,h}\left\|\xi^k_{u,h+1} - \chi^k_{u,h+1}\right\|^2 
    \nonumber\\
    &
    \overset{(a)}{=} \Big\| \big(I - \gamma_{\xi,H}A(\theta_k) \big)\big(\xi^k_{u,h} - \chi^k_{u,h}\big)\Big\|^2 
    + \gamma^2_{\xi,H}\bE_{k,h}\big\| A(\theta_k)- \hA_\MLMC \big\|^2\left\|\xi^k_{u,h} - \chi^k_{u,h}\right\|^2 + \gamma^2_{\xi,H}\bE_{k,h}\big\| \hV_{u,\MLMC} \big\|^2
    \nonumber\\
    & \quad + 2\gamma^2_{\xi,H}\bE_{k,h}\left\langle A(\theta_k) - \hA_\MLMC, \hV_{u,\MLMC} \right\rangle
    + 2\gamma_{\xi,H}\Big\langle A(\theta_k) - \bE_{k,h}[\hA_\MLMC],
    \big(I - \gamma_{\xi,H}A(\theta_k) \big)\left(\xi^k_{u,h} - \chi^k_{u,h}\right) \Big\rangle
    \nonumber\\
    &
    \overset{(b)}{\leq} \left( 1- \mu_\phi\gamma_{\xi,H}\right)^2\big\|\xi^k_{u,h} - \chi^k_{u,h}\big\|^2 
    + 2\gamma^2_{\xi,H}\left\|\xi^k_{u,h} - \chi^k_{u,h}\right\|^2\bE_{k,h}\big\| A(\theta_k)- \hA_\MLMC \big\|^2 + 2\gamma^2_{\xi,H}\bE_{k,h}\big\| \hV_{u,\MLMC} \big\|^2
    \nonumber\\
    & \quad + 2\gamma^2_{\xi,H}\bE_{k,h}\big\| A(\theta_k)- \hA_\MLMC \big\|^2 + 2\gamma_{\xi,H}\Big\langle A(\theta_k) - \bE_{k,h}[\hA_\MLMC],
    \big(I - \gamma_{\xi,H}A(\theta_k) \big)\left(\xi^k_{u,h} - \chi^k_{u,h}\right) \Big\rangle
    \nonumber\\
    & \overset{(c)}{\leq}  \left( 1- 2\mu_\phi\gamma_{\xi,H} + \mu^2_\phi\gamma^2_{\xi,H} \right)\big\|\xi^k_{u,h} - \chi^k_{u,h}\big\|^2 + 2\gamma^2_{\xi,H}\left( \sigma^2_A(1+R^2) + \sigma^2_V \right)
    + 2\gamma_{\xi,H}\left(1-\mu_\phi\gamma_{\xi,H} \right)\big\|\xi^k_{u,h} - \chi^k_{u,h}\big\|\delta_A
    \nonumber\\
    &\overset{(d)}{\leq} \left( 1- \frac{\mu_\phi}{2}\gamma_{\xi,H} \right)\big\|\xi^k_{u,h} - \chi^k_{u,h}\big\|^2 + 2\gamma^2_{\xi,H}\left( \sigma^2_A(1+R^2) + \sigma^2_V \right) + \delta^2_A/\mu_\phi,
\end{align}
where $(a)$ follows since $\xi^k_{u,h}, \chi^k_{u,h}$ are $\cF_{k,h}$-measurable, $(b)$ follows since assuming $\ker(A(\theta_k))=\{0\},$ (this holds if the vectors $\{\nabla_\zeta Q(s,a,\zeta_0)\}$ do not span the all ones vector in $\bR^{\SA}$) we can show that for $ c_\gamma$ sufficiently large,
$A(\theta_k) \succeq \mu_\phi I$ (see \cite[Lemma 17]{Satheesh2026GlobalCO}), and using the inequality $2xy\leq x^2+y^2,$ and $(c)$ follows since $\xi^k_{u,H},\chi^k_{u,H}$ lie in the NTK ball and by using the bounds \eqref{e: linear.critic.MSE.bounds} and \eqref{e: linear.critic.bias.bounds}. Finally, $(d)$ follows as we apply Young's inequality, choose $\gamma_{\xi,H}<\frac{1}{2\mu_\phi}$ the inequality $1-x<e^{-x}$

Then, applying~\eqref{e: MSE.diff.recursive} recursively and following the arguments in the proof of~\cite[Theorem B.1]{xu2026global}, we get 
\begin{equation}
    \bE_k\big\|\xi^k_{u,H} - \chi^k_{u,H}\big\|^2 
    = \cO\Big( \bE_k\big\|\xi^k_{u,0}- \chi^k_{u,0}\big\|^2 e^{-\mu_\phi H\gamma_{\xi,H}}  
    + \gamma_{\xi,H}\left(\sigma^2_A(1+R^2) + \sigma^2_V \right) + \delta^2_A\Big).
\end{equation}
Taking $R=\varTheta(\ln T)$ and  $\gamma_{\xi,H}=\frac{2\ln T}{\mu_\phi H},$ and using the fact that $\xi^k_{u,0}=\chi^k_{u,0}$,
\begin{equation}\label{e: critic.linear.diff}
    \bE_k\big\|\xi^k_{u,H} - \chi^k_{u,H}\big\|^2 = \widetilde{\cO}\left( \frac{1}{H}\left(\sigma^2_A + \sigma^2_V \right) + \delta^2_A   \right),
\end{equation}
Taking the NTK radius $R$ sufficiently large so that $\xi^*_{u,k}$ lies in the NTK ball, we apply the same recipe used to obtain~\eqref{e: critic.linear.diff} to the auxiliary update~\eqref{e: linear.critic.update.decomposed}. This gives us
\begin{equation}\label{e: critic.linear}
    \bE_k\big\|\chi^k_{u,H} - \xi^*_{u,k}\big\|^2 
    = \widetilde{\cO}\bigg( \frac{1}{H}\left(\sigma^2_A + \sigma^2_B \right) 
    + \delta^2_A + \delta^2_B + \sigma^2_V \bigg).
\end{equation}
Hence, to conclude Lemma~\ref{lem: vanilla.critic.bound}, it remains to bound the MSEs and biases of $\hA_\MLMC,$ $\hB_{u,\MLMC},$ and $\hV_{u,\MLMC}.$
By Lemma~\ref{app.lem: NTK.bounds} and~\eqref{e: linear.critic.operator}, the corresponding single-sample biases satisfy
\begin{multline}
    \big\|\bE_{\theta_k}[\hA(z,z')] - A(\theta_k) \big\| = 0,
    \quad 
    \big\|\bE_{\theta_k}[\hB_u(z,z')] - B_u(\theta_k) \big\| = 0,
    \\
    \text{and} \qquad \big\|\bE_{\theta_k}[\hV_u(z,z')]\big\| = \widetilde{\cO}\left(R\frac{\sqrt{\ln\left(\Hm/\varepsilon\right)}}{\sqrt{m}}\right),
\end{multline}
while under NTK regime, the single-sample MSEs satisfy
\begin{multline}
    \big\|\hA(z,z') - \bE_{\theta_k}[\hA(z,z')] \big\|^2 = \cO(c^2_\gamma),
    \quad \big\|\hB_u(z,z') - \bE_{\theta_k}[\hB_u(z,z')] \big\| = \widetilde{\cO}(c^2_\gamma),
    \\
    \text{and} \qquad \big\|\hV_u(z,z')- \bE_{\theta_k}[\hV_u(z,z')]\big\| = \widetilde{\cO}\left(R\frac{\sqrt{\ln\left(\Hm/\varepsilon\right)}}{\sqrt{m}}\right).
\end{multline}
Applying Lemma~\ref{app.lem: MLMC}, we conclude $\sigma^2_V = \widetilde{\cO}(\tmix \ln(\Hm/\varepsilon)m^{-1}) $ and
\begin{equation}
    \sigma^2_A = \cO(\tmix\ln \Tm),\quad  \delta^2_A = \cO(\tmix\Tm^{-1}), \quad
    \sigma^2_B = \cO(\tmix\ln \Tm),\quad \delta^2_B = \cO( \tmix\Tm^{-1}).
\end{equation}
Finally, we substitute the above bounds into~\eqref{e: critic.linear.diff} and~\eqref{e: critic.linear}, and use the inequality
\begin{equation}
    \bE_k\big\|\xi^k_{u,H} - \xi^*_{u,k}\big\|^2 
    \leq 2\bE_k\big\|\xi^k_{u,H} - \chi^k_{u,H}\big\|^2 + 2\bE_k\big\| \chi^*_{u,k} - \xi^*_{u,k}\big\|^2.
\end{equation}
This completes the proof of Lemma~\ref{lem: vanilla.critic.bound}. 
\end{proof}

Next, we conclude this section with the proof of Theorem~\ref{thm: critic.bounds.MLMC}. 

\begin{proof}[\textbf{Proof of Theorem~\ref{thm: critic.bounds.MLMC}}]
By the choice
$H_0=1+\floor{4\ln T}$, we have
\[
    \gamma_{\xi,H_0}
    =\frac{2\ln T}{\mu_\varphi H_0}
    <\frac{1}{2\mu_\varphi}.
\]
Since every vanilla-critic estimate used by the outer MLMC layer has horizon $H\geq H_0$, we have
$\gamma_{\xi,H}\leq\gamma_{\xi,H_0}<1/(2\mu_\phi)$.
Thus, under the remaining hypotheses of
Lemma~\ref{lem: vanilla.critic.bound}, its convergence
guarantee applies to every vanilla output appearing in
the outer MLMC layer.

Now, recall from~\eqref{e: function .critic.MLMC}
and~\eqref{e: output.critic.MLMC} that
\begin{align*}
    Q_{u,k,\MLMC}(\cdot,\cdot)
    &:= Q\left(\cdot,\cdot,\zeta^k_{u,H_0}\right)
    + \ones_{\{H_0 2^{H_k}\leq \Hm\}}2^{H_k}
    \left[
    Q\big(\cdot,\cdot,\zeta^k_{u,H_0 2^{H_k}}\big)
    - Q\big(\cdot,\cdot,\zeta^k_{u,H_0 2^{H_k-1}}\big)
    \right],
\end{align*}
and
\begin{align*}
    \xi_{u,k,\MLMC}
    := \xi^k_{u,H_0}
    + \ones_{\{H_0 2^{H_k}\leq \Hm\}}2^{H_k}
    \left(
    \xi^k_{u,H_0 2^{H_k}}
    - \xi^k_{u,H_0 2^{H_k-1}}
    \right),
\end{align*}
where $H_k\sim\textnormal{Geom}(1/2).$ Assume $\Hm\geq H_0$ and define
$N:= H_0\ 2^{\floor{\log_2(\Hm/H_0)}}$. Then, we have
\begin{align}
    &\bE_k\big[Q_{u,k,\MLMC}(\cdot,\cdot)\big]
    -\bE_k\left[Q\big(\cdot,\cdot,\zeta^k_{u,H_0}\big)\right]
    \nonumber\\
    & =
    \sum_{h=1}^{\floor{\log_2(\Hm/H_0)}}
    2^h\cdot 2^{-h}
    \left(
    \bE_k\left[Q\big(\cdot,\cdot,\zeta^k_{u,H_0 2^h}\big)\right]
    -\bE_k\left[Q\big(\cdot,\cdot,\zeta^k_{u,H_0 2^{h-1}}\big)\right]
    \right)
    \nonumber\\
    &= \bE_k\left[Q\big(\cdot,\cdot,\zeta^k_{u,N}\big)\right]
    -\bE_k\left[Q\big(\cdot,\cdot,\zeta^k_{u,H_0}\big)\right],
\end{align}
Hence,
\begin{equation}\label{e: bias.MLMC.critic.function}
    \bE_k\big[Q_{u,k,\MLMC}(\cdot,\cdot)\big]
    =\bE_k\left[Q\big(\cdot,\cdot,\zeta^k_{u,N}\big)\right].
\end{equation}
Since $\Hm/2<N\leq\Hm$, the finest represented horizon
remains of order $\varTheta(\Hm)$.

A similar telescoping argument gives
\begin{equation}\label{e: bias.MLMC.critic.parameter}
    \bE_k\big[\xi_{u,k,\MLMC}\big]= \bE_k\left[\xi^k_{u,N}\right], \quad \text{where} \quad N=H_0 2^{\floor{\log_2(\Hm/H_0)}}
\end{equation}
Thus, the critic-function bias satisfies
\begin{multline}
    \delta_Q^2
    = \left\|
    \bE_k\Big[
    Q_{u,k,\MLMC}(\cdot,\cdot)
    - \Ql(\cdot,\cdot,\zeta_{u,k,\MLMC})
    \Big]\right\|^2
    \\
    \overset{(a)}{=}
    \left\|\bE_k\Big[
    Q\big(\cdot,\cdot,\zeta^k_{u,N}\big)
    - \Ql\big(\cdot,\cdot,\zeta^k_{u,N}\big)
    \Big]\right\|^2
    \overset{(b)}{=} \widetilde{\cO}\left(\frac{\ln(\Hm/\varepsilon)}{m} \right),
\end{multline}
where $(a)$ follows from
\eqref{e: bias.MLMC.critic.function}
and~\eqref{e: bias.MLMC.critic.parameter},
using the affine dependence of $\Ql(\cdot,\cdot,\zeta)$
on $\zeta$. Step $(b)$ follows from Jensen's inequality
and the NTK bounds in Lemma~\ref{app.lem: NTK.bounds},
since $\zeta^k_{u,N}\in\cB$.

Further, the critic-parameter bias satisfies
\begin{align}
    \left\|\bE_k[\xi_{u,k,\MLMC}]-\xi^*_{u,k}\right\|^2
    &\overset{(a)}{=}
    \left\|\bE_k[\xi^k_{u,N}]-\xi^*_{u,k}\right\|^2
    \nonumber\\
    &\overset{(b)}{\leq}
    \bE_k\left\|\xi^k_{u,N}-\xi^*_{u,k}\right\|^2
    \\
    &\overset{(c)}{=}
    \widetilde{\cO}\!\left(
    \frac{1}{T^2} +\frac{\tmix}{\Tm}+\frac{\tmix}{N}+\frac{\tmix}{m}
    \right)
    \nonumber\\
    &=\widetilde{\cO}\!\left(
    \frac{1}{T^2} +\frac{\tmix}{\Tm} +\frac{\tmix}{\Hm} +\frac{\tmix\ln(\Hm/\varepsilon)}{m} \right),
\end{align}
where $(a)$ follows from
\eqref{e: bias.MLMC.critic.parameter},
$(b)$ follows from conditional Jensen's inequality,
and $(c)$ follows from Lemma~\ref{lem: vanilla.critic.bound},
since $N\geq H_0$. The last equality uses $\Hm/2<N\leq\Hm$.
Under $\Hm=\Theta(T)$, the $T^{-2}$ term is absorbed,
yielding
\begin{equation}
    \left\|\bE_k[\xi_{u,k,\MLMC}]-\xi^*_{u,k}\right\|^2
    =\widetilde{\cO}\!\left( \frac{\tmix}{\Tm} +\frac{\tmix}{\Hm} +\frac{\tmix\ln(\Hm/\varepsilon)}{m} \right).
\end{equation}
        
Now, we bound the critic-parameter MSE. Using the same telescoping sum as earlier, we have
\begin{align}\label{e: MSE.critic.MLMC.param}
    &\bE_k\big\|\xi_{u,k,\MLMC}-\xi^*_{u,k}\big\|^2
    \nonumber\\
    &\overset{(a)}{\leq}
    2\bE_k\big\|\xi^k_{u,H_0}-\xi^*_{u,k}\big\|^2
    +2\sum_{h=1}^{\floor{\log_2(\Hm/H_0)}}4^h\,2^{-h}
    \bE_k\big\|
    \xi^k_{u,H_0 2^h}-\xi^k_{u,H_0 2^{h-1}}
    \big\|^2
    \nonumber\\
    &\overset{(b)}{\leq}
    2\bE_k\big\|\xi^k_{u,H_0}-\xi^*_{u,k}\big\|^2
    +4\sum_{h=1}^{\floor{\log_2(\Hm/H_0)}}2^h
    \left(
    \bE_k\big\|\xi^k_{u,H_0 2^h}-\xi^*_{u,k}\big\|^2
    +\bE_k\big\|\xi^k_{u,H_0 2^{h-1}}-\xi^*_{u,k}\big\|^2
    \right)
    \nonumber\\
    &\overset{(c)}{=} \widetilde{\cO}\!\left(
    \frac{\tmix(1+\floor{\log_2(\Hm/H_0)})}{H_0} +2^{\floor{\log_2(\Hm/H_0)}} \left[\frac{1}{T^2}+\frac{\tmix}{\Tm}+\frac{\tmix\ln(\Hm/\varepsilon)}{m}\right]\right)
    \nonumber\\
    &=\widetilde{\cO}\!\left(
    \frac{\tmix}{H_0}
    +\frac{\Hm}{H_0}
    \left[
    \frac{1}{T^2}+\frac{\tmix}{\Tm}+\frac{\tmix\ln(\Hm/\varepsilon)}{m}
    \right]
    \right),
\end{align}
where $(a)$ uses~\eqref{e: output.critic.MLMC},
$\bP(H_k=h)=2^{-h}$, and the inequality
$\|x+y\|^2\leq2\|x\|^2+2\|y\|^2$.
Step $(b)$ applies the same inequality after adding and
subtracting $\xi^*_{u,k}$, and $(c)$ follows from
Lemma~\ref{lem: vanilla.critic.bound}, since every
horizon is at least $H_0$.
The last line uses $2^{\floor{\log_2(\Hm/H_0)}}\leq\Hm/H_0$
and absorbs logarithmic factors.
Under $m=\Hm=\Tm=\Theta(T)$ and
$H_0=1+\floor{4\ln T}=\Theta(\ln T)$
for fixed $\mu_\varphi>0$, this yields
$\bE_k\|\xi_{u,k,\MLMC}-\xi^*_{u,k}\|^2
=\widetilde{\cO}(1).$
    
Finally, we bound the critic-function MSE. Define the
linearization remainder
\begin{equation}
    R^k_{u,H}:=Q(\cdot,\cdot,\zeta^k_{u,H})
    -\Ql(\cdot,\cdot,\zeta^k_{u,H}).
\end{equation}
Since $\Ql$ is affine in its parameter, the outer MLMC
definitions imply
\begin{align}
    & Q_{u,k,\MLMC}
    -\Ql(\cdot,\cdot,\zeta_{u,k,\MLMC}) =R^k_{u,H_0}
    +\ones_{\{H_0 2^{H_k}\leq\Hm\}}2^{H_k}
    \left(
    R^k_{u,H_0 2^{H_k}} -R^k_{u,H_0 2^{H_k-1}} \right).
\end{align}
Thus, we have
\begin{align}
    \bE_k\big\| Q_{u,k,\MLMC} -\Ql(\cdot,\cdot,\zeta_{u,k,\MLMC}) \big\|^2
    & \leq
    2\bE_k\|R^k_{u,H_0}\|^2
    +2\sum_{h=1}^{L_{\max}}2^h
    \bE_k\big\|
    R^k_{u,H_0 2^h}-R^k_{u,H_0 2^{h-1}}
    \big\|^2
    \nonumber\\
    &\leq
    2\bE_k\|R^k_{u,H_0}\|^2
    +4\sum_{h=1}^{\floor{\log_2(\Hm/H_0)}}2^h
    \left(
    \bE_k\|R^k_{u,H_0 2^h}\|^2
    +\bE_k\|R^k_{u,H_0 2^{h-1}}\|^2
    \right)
    \nonumber\\
    &=\widetilde{\cO}\!\left(
    \frac{1}{m}\left[1+\sum_{h=1}^{\floor{\log_2(\Hm/H_0)}}2^h\right]
    \right)
    =\widetilde{\cO}\!\left(\frac{\Hm\ln(\Hm/\varepsilon)}{H_0m}\right).
\end{align}
The first two inequalities use
$\|x+y\|^2\leq2\|x\|^2+2\|y\|^2$
and $\bP(H_k=h)=2^{-h}$.
The NTK bounds in Lemma~\ref{app.lem: NTK.bounds}
give $\bE_k\|R^k_{u,H}\|^2=\widetilde{\cO}(\ln(\Hm/\varepsilon)/m)$
at every represented horizon, since the corresponding critic
parameters lie in the NTK ball $\cB$.
The last equality uses $2^{\floor{\log_2(\Hm/H_0)}}\leq\Hm/H_0 \leq \Hm$.
Consequently,
\[
    \sigma_Q^2
    =\widetilde{\cO}\!\left(\frac{\Hm\ln(\Hm/\varepsilon)}{H_0m}\right).
\]
This completes the proof of
Theorem~\ref{thm: critic.bounds.MLMC}.
\end{proof}

\section{Proof of Theorem~\ref{thm: main result}}
\label{appendix: main.result.proof}

We finally combine the bounds derived in Theorems~\ref{thm: NPG.bounds} and~\ref{thm: critic.bounds.MLMC}, and Lemmas~\ref{lem: Lagrange.error},~\ref{lem: constraint.violation.rate}, and~\ref{lem: optimality.gap.bound} to prove Theorem~\ref{thm: main result}.

We choose $H_0=1+\floor{4\ln T}$ and $m=\Hm=\Tm=T,$ as prescribed in Theorem~\ref{thm: main result}. Then, Lemma~\ref{thm: critic.bounds.MLMC} gives the following critic bounds: 
\begin{equation}\label{e: critic.final.func}
    \sigma^2_Q = \widetilde{\cO}\left(\ln(1/\varepsilon)\right)   \quad \text{and} \quad  \delta^2_Q = \widetilde{\cO}\left( \ln(1/\varepsilon)/T\right),
\end{equation}
and 
\begin{equation}\label{e: critic.final.param}
    \bE_k\left\| \xi_{u,k,\MLMC} - \xi^*_{u,k} \right\|^2 = \widetilde{\cO}\left(\ln(1/\varepsilon)\right)
    \quad \text{and} \quad 
    \left\|\bE_k\big[\xi_{u,k,\MLMC}\big] - \xi^*_{u,k}\right\|^2 = \widetilde{\cO}\left( \tmix\ln(1/\varepsilon)/T\right).
\end{equation}
Plugging the bounds from~\eqref{e: critic.final.func} and~\eqref{e: critic.final.param} into Theorem~\ref{thm: NPG.bounds} and taking $\bH=\varTheta(\ln^2 T),$ we get
\begin{equation}\label{e: NPG.final}
        \bE_k\big\| w_k - \wS_k \big\|^2 
        = \widetilde{\cO}\Big( \tmix\ln(1/\varepsilon) + \epsap \Big)
        \quad \text{and} \quad 
        \big\| \bE_k[w_k] - \wS_k\big\|^2
        = \widetilde{\cO}\left( \epsap  + \frac{\tmix\ln(1/\varepsilon)}{T} \right).
\end{equation}
Next, we substitute the above NPG bias and MSE bounds into Lemma~\ref{lem: Lagrange.error} and choose $\alpha=1/\sqrt{K}$ and $K=\varTheta(T)$ to get the following Lagrange error:
\begin{align}\label{e: Lagrange.final}
    \frac{1}{K}\sum_{k=0}^{K-1}\bE\big[\cL(\piS, \lambda_k) - \cL(\theta_k,\lambda_k) \big] 
    &= \cO \bigg( \sqrt{\epsb} + \sqrt{\epsap} + \tmix\alpha\ln(1/\varepsilon) + \frac{1}{\alpha K} + \frac{\sqrt{\tmix\ln(1/\varepsilon)}}{\sqrt{T}}  \bigg)
    \nonumber\\
    & = \widetilde{\cO} \bigg( \sqrt{\epsb} + \sqrt{\epsap} + \frac{\tmix\ln(1/\varepsilon)}{\sqrt{T}}  \bigg).
\end{align}
Finally, we plug the bounds from~\eqref{e: Lagrange.final} and~\eqref{e: critic.final.param} into Lemma~\ref{lem: optimality.gap.bound} and choose $\beta=1/\sqrt{T}.$ This gives us the following averaged error:
\begin{equation}
    J^{\piS}_r -\frac{1}{K}\sum_{k=0}^{K-1}\bE\big[\Jthk_r \big] 
    = \widetilde{\cO} \bigg( \sqrt{\epsb} + \sqrt{\epsap} + \frac{\tmix\ln(1/\varepsilon)}{\sqrt{T}} \bigg).
\end{equation}
Likewise, plugging~\eqref{e: Lagrange.final} and~\eqref{e: bias.MLMC.critic.parameter} into Lemma~\ref{lem: constraint.violation.rate} gives the constraint violation rate:
\begin{equation}
    -\frac{1}{K} \sum_{k=0}^{K-1} \bE\left[\Jthk_c \right] = \widetilde{\cO} \bigg( \sqrt{\epsb} + \sqrt{\epsap} + \frac{\tmix\ln(1/\varepsilon)}{\sqrt{T}} \bigg).
\end{equation}
This concludes the proof of Theorem~\ref{thm: main result}. \hfill\qed

\section{Technical lemmas}
\label{appendix: technical.lemmas}

The following lemma is the standard performance difference lemma~\cite[Lemma 4]{bai2024regret} for average-reward MDPs.
\begin{lemma}[\textbf{Performance Difference Lemma}]\label{app.lem: performance.diff}
    For any two policies $\pth$ and $\pi_{\theta'},$ and $u\in \{r,c\},$ we have
    \[        
        \Jp_u - J^{\pi'}_u \leq \ \bE_{(s,a)\sim \nu^\pi}[A_u^{\theta'}(s,a)],
    \]
    where $A^{\theta'}_u$ is the advantage function of policy $\pi_{\theta'}$ as defined in~\eqref{e: advantage}.
\end{lemma}

The following lemma, proved in~\cite[Lemma 15]{Satheesh2026GlobalCO},
collects several useful bounds for neural critics under the NTK regime.
\begin{lemma}[\textbf{NTK bounds}]\label{app.lem: NTK.bounds}
    Choose $R=\cO(\ln T)$ and let $\xi^k_{u,h}\in \cB,$ for all $0\leq h<H,$  Then, the following holds with probability at least $(1-\varepsilon -2Le^{-Cm}):$ 
    
    \begin{enumerate}[label = (\alph*)]
        \item $\left\|\nabla_\zeta Q(\cdot,\cdot,\zeta)\right\| \leq C_1;$ $\ \left|Q(\cdot,\cdot,\zeta)\right| =\cO\big(\sqrt{\ln(H/\varepsilon)}\big);$ 
        \item $\left\|Q(\cdot,\cdot,\zeta) - \Ql(\cdot,\cdot,\zeta) \right\| =\cO\Big(\frac{\sqrt{\ln(H/\varepsilon)}}{\sqrt{m}}\Big);$
        \item $\left\|\nabla_\zeta Q(\cdot,\cdot,\zeta) - \nabla_\zeta Q(\cdot,\cdot,\zeta_0)\right\| =\cO\Big(\frac{\sqrt{\ln(H/\varepsilon)}}{\sqrt{m}}\Big).;$
        \item $\|\hV_u(\cdot,\cdot,\cdot,\cdot)\| = \cO\Big(\frac{\sqrt{\ln(H/\varepsilon)}}{\sqrt{m}}\Big).$
    \end{enumerate}
    
\end{lemma}

The following lemma, proved in~\cite[Theorem 2]{pmlr-v267-ganesh25b},
characterizes the bias and MSE of a stochastic linear recursion.
\begin{lemma}[\textbf{Stochastic Linear Recursion}]\label{app.lem: gen.linear}
    Consider the linear stochastic recursion
    \begin{equation*}
        x_{h+1} = x_h + \gamma\big[\hat{q}_h - \hat{P}_h x_h \big],
    \end{equation*}
    where $(\hat{q}_h, \hat{P}_h)_{h=0}^{H-1}$ are noisy estimates for $(q, P)\in \bR^d\times\bR^{d\times d}.$ The target is an $\xS$ such that $P\xS=q.$ Suppose that $(\hat{q}_h,\hat{P}_h)$ satisfies:
    \begin{align*}
        \bE_h\|\hat{P}_h - P \|^2 \leq \sigma^2_P,& \quad \bE_h\|\hat{q}_h - q \|^2 \leq \sigma^2_q, 
        \\
        \|\bE_h[\hat{P}_h] - P \|^2 \leq \delta^2_P,&  \quad  \|\bE_h[\hat{q}_h] - q \|^2 \leq \delta^2_q.
    \end{align*}
    Further, let
    \begin{equation*}
        \|\bE[\hat{q}_h] - q \|^2 \leq \bar{\delta}^2_q, \quad \|P\|\leq \Lambda_P,
        \quad
        \|\hat{q}_h\|\leq \Lambda_q, \quad \lambda_PI \preceq P \quad \text{such that} \quad 0< \gamma(24\sigma^2_P + 8\Lambda^2_P) \leq \lambda_P.
    \end{equation*}
    Then, after $H$ iterations, we have
    \begin{align*}
        \bE\|x_H - \xS\|^2  & \leq \|x_0 - \xS\|^2\ e^{-\gamma\lambda_P H} + \cO\left( \gamma R_0 + R_1 \right)
        \\
        \|\bE[x_H] - \xS \|^2 
        &\leq \|x_0 - \xS\|^2\big(e^{-\bbeta\lambda_P H} + \delta^2_P \big)
        + \cO\Big( \delta^2_P\left( \gamma R_0 +  R_1 \right)  + \bar{R}_1 \Big),
    \end{align*}
    where 
    \begin{equation*}
        R_0 = \cO\left(\sigma^2_P + \sigma^2_q \right), \quad R_1 = \cO\left(\delta^2_P + \delta^2_q  \right),\quad
        \text{ and } \quad \bar{R}_1 = \cO\left(\delta^2_P + \bar{\delta}^2_q \right).
    \end{equation*}
\end{lemma}

The following lemma is adapted for biased estimates from~\cite[Lemma 3.1]{dorfman2022adapting}, and summarizes the computational cost, bias, and MSE guarantees of MLMC estimation.
\begin{lemma}[\textbf{MLMC bounds}]\label{app.lem: MLMC}
    Let $(Z_t)$ be a time-homogeneous ergodic Markov chain with a unique invariant distribution $d_Z,$ and a mixing-time $\tmix.$ Assume that $F(x,Z)$ is an estimate of $F(x).$ Let for all $t\geq 0,$
    \begin{equation*}
        \|\bE_{d_Z}[F(x,Z)] - F(x)\|^2\leq \delta^2 \quad
        \text{and} \quad \bE\| F(x,Z_t) - \bE_{d_Z}[F(x,Z)]\|^2 \leq \sigma^2.
    \end{equation*}
    Further, let $Q\sim \textnormal{Geom(1/2)}.$ Then, the MLMC estimator defined as 
    \[
        g_\MLMC := g_1 + \ones_{2^Q\leq \Tm}\cdot 2^Q\left( g_{2^Q} - g_{2^{Q-1}} \right),
    \]
    where $g_j:= j^{-1}\sum_{t=0}^{j-1}F(x,Z_t),$ satisfies the following bounds.
    \begin{enumerate}[label = (\alph*)]
        \item $\bE[g_\MLMC] = \bE[g_{\floor{\log_2\Tm}}];$
        \item $\bE\| F(x) - g_\MLMC\|^2 = \cO\left( \sigma^2\tmix\log_2\Tm + \delta^2 \right);$
        \item $\|F(x) - \bE[g_\MLMC]\|^2 = \cO\left( \sigma^2\tmix \Tm^{-1} + \delta^2 \right).$
    \end{enumerate}
\end{lemma}

%


\end{document}